\documentclass{article}

\usepackage{tgpagella}
\usepackage[T1]{fontenc}
\usepackage{todonotes}
\usepackage{graphicx}
\usepackage{etex}
\usepackage{enumitem}
\setlist{parsep = -0em, itemsep = 0.25em}
\usepackage[utf8]{inputenc}
\usepackage[dvips,letterpaper,margin=1in]{geometry}
\usepackage{amsmath}
\usepackage{amssymb}
\usepackage{amsthm}
\usepackage{bm}
\usepackage{colonequals}
\usepackage{comment}
\usepackage{graphicx}
\usepackage{subcaption}
\usepackage{xcolor}
\usepackage{tikz} 
\usetikzlibrary{shadows} 
\usepackage{authblk}
\usepackage[utf8]{inputenc} 
\usepackage[T1]{fontenc}    
\usepackage{url}            
\usepackage{booktabs}       
\usepackage{amsfonts}       
\usepackage{nicefrac}       
\usepackage{microtype}      
\usepackage{hyperref}
\hypersetup{colorlinks = true,
            linkcolor = blue,
            urlcolor  = teal,
            citecolor = magenta,
            anchorcolor = blue}

\usepackage{amsmath,amsthm,amssymb}
\usepackage{algorithm,algpseudocode}
\usepackage{wrapfig}
\usepackage{esint}

\newtheorem{theorem}{Theorem}[section]
\newtheorem{proposition}[theorem]{Proposition}
\newtheorem{lemma}[theorem]{Lemma}

\DeclareMathOperator{\argmin}{argmin}

\title{A Functional Perspective on Learning Symmetric Functions with Neural Networks}

\author[a]{Aaron Zweig}
\author[a,b]{Joan Bruna \thanks{This work is partially supported by the Alfred P. Sloan Foundation, NSF RI-1816753, NSF CAREER CIF 1845360, and the Institute for Advanced Study.}}

\affil[a]{Courant Institute of Mathematical Sciences, New York
  University, New York}
\affil[b]{Center for Data Science, New York University}

\usepackage{amsmath,amsthm,amssymb}
\usepackage{bbm}
\usepackage{graphicx}
\usepackage{caption}
\usepackage{subcaption}
\usepackage{todonotes}

\usepackage{wrapfig}

\usepackage[font=small,labelfont=bf]{caption}

\usepackage{tikz}
\usetikzlibrary{automata,positioning}
\usetikzlibrary{shapes,snakes}
\usetikzlibrary{decorations.pathmorphing}

\newcommand{\N}{\mathbb{N}}

\def\FF{\mathcal{F}} 
\def\A{\mathcal{A}} 
\def\S{\mathcal{S}}
\def\R{\mathbb{R}}
\def\E{\mathbb{E}}
\def\I{\mathbb{I}}
\def\F{\mathcal{F}_{\mathrm{sym}}}

\def\PI{\mathcal{P}(\I)}

\def\xb{\boldsymbol{x}}

\usepackage[style=alphabetic]{biblatex}
\renewcommand*\parencite[1]{\cite{#1}}
\renewcommand*\textcite[1]{\cite{#1}}

\usepackage[utf8]{inputenc} 
\usepackage[T1]{fontenc}    
\usepackage{hyperref}       
\usepackage{url}            
\usepackage{booktabs}       
\usepackage{amsfonts}       
\usepackage{nicefrac}       
\usepackage{microtype}      

\begin{document}

\maketitle

\begin{abstract}
Symmetric functions, which take as input an unordered, fixed-size set, are known to be universally representable by neural networks that enforce permutation invariance.  These architectures only give guarantees for fixed input sizes, yet in many practical applications, including point clouds and particle physics, a relevant notion of generalization should include varying the input size.  In this work we treat symmetric functions (of any size) as functions over probability measures, and study the learning and representation of neural networks defined on measures.  By focusing on shallow architectures, we establish approximation and generalization bounds under different choices of regularization (such as RKHS and variation norms), that capture a hierarchy of functional spaces with increasing degree of non-linear learning. The resulting models can be learned efficiently and enjoy generalization guarantees that extend across input sizes, as we verify empirically. 
  
  
\end{abstract}

\section{Introduction}

Deep learning becomes far more efficient with prior knowledge of function invariants.  This knowledge underlies architectural choices that enforce the invariance or equivariance in the network, including Convolutional Neural Networks \parencite{lecun1998gradient} which encode translation symmetries, and Graph Neural Networks \parencite{scarselli2008graph} which encode conjugate permutation symmetries.  For functions with invariance to permutation of the input elements, several universal architectures encode this invariance by treating the input as a set~\parencite{zaheer2017deep, qi2017pointnet}.  However, these formulations assume a constant input size, which precludes learning an entire family of symmetric functions.

Such symmetric functions appear naturally across several domains, including particle physics, computer graphics, population statistics and cosmology. Yet, in most of these applications, the input size corresponds to a sampling parameter that is independent of the underlying symmetric function of interest. As a motivating example, consider the function family induced by the max function, where for varying $N$, $f_N(\{x_1\dots x_N\}) = \max_{i\leq N} x_i$.  It is natural to ask if a network can simultaneously learn all these functions.

In this work, we interpret input sets as an empirical measure defined over the base space $\I$, and develop families of neural networks defined over the space of probability measures \emph{probability measures of $\I$}, as initially suggested in \textcite{pevny2019approximation,de2019stochastic}. We identify functional spaces characterized by neural architectures and provide generalization bounds that showcase a natural hierarchy among spaces of symmetric functions. In particular, our framework allows us to understand the question of generalizing across input sizes as a corollary. Our constructions rely on the  theory of infinitely wide neural networks \parencite{bengio2006convex, rosset2007,bach2017breaking}, and provide a novel instance of \emph{depth separation} leveraging the symmetric structure of the input. 

\paragraph{Summary of Contributions:} 
We consider the infinite-width limit of neural networks taking as domain the space of probability measures in order to formalize learning of symmetric function families.  We prove a necessary and sufficient condition for which symmetric functions can be learned.  By controlling the amount of non-linear learning, we partition the space of networks on measures into several function classes, proving a separation result among the classes as well as proving a generalization result and empirically studying the performance of these classes to learn symmetric functions on synthetic and real-world data.



\paragraph{Related Work}

Several works consider representing symmetric functions of fixed input size with invariant neural networks, and in particular there are two main universal architectures, DeepSets~\parencite{zaheer2017deep} and PointNet~\parencite{qi2017pointnet}.  An alternative generalization of DeepSets is given in~\textcite{maron2019universality}, which proves the universality of tensor networks invariant to any subgroup of the symmetric group.  Regarding variable input size, the work from~\textcite{wagstaff2019limitations} proves lower bounds on representation of the max function in the DeepSets architecture with a dependency on input size.

Separately, there is a wide literature considering neural networks that act on elements on functional data.  These results mainly consider universal approximation~\textcite{sandberg1996network, stinchcombe1999neural, rossi2005functional}.  The work~\textcite{mhaskar1997neural} bears some similarity to the present work, as they prove a quantitative separation between the class of neural networks and the class of functionals with bounded norm, while our main result shows separations among several neural network classes.

The work most similar to ours are~\textcite{pevny2019approximation,de2019stochastic}, which also normalize the DeepSets architecture to define a function on measures.  However, they only prove the universality of this model, while we justify the model by classifying symmetric families that are representable and recovering generalization results.  We also build on the framework given by~\textcite{bach2017breaking}, which introduces function classes to characterize neural networks in the wide limit, and proves statistical generalization bounds to demonstrate the advantage of non-linear learning.  Although we motivate our work from symmetric functions on finite sets, there are applications in multi-label learning~\parencite{frogner2015learning} and evolving population dynamics~\parencite{hashimoto2016learning} that require functions of measures.

\paragraph{Roadmap:} We introduce notation and summarize the infinite-dimensional convex network theory theory~\parencite{bach2017breaking} in Section~\ref{sec:prelim}.  In Section~\ref{sec:set2measure} we introduce measure networks and characterize their relationship to symmetric functions.  Section~\ref{sec:approx} gives our main result, separating several classes of measure networks according to the degree of non-linear learning, and Section~\ref{sec:gen} introduces generalization results.  Finally, we detail several experiments with the finite instantiation of measure networks in Section~\ref{sec:experiments}.
\section{Preliminaries}\label{sec:prelim}

\subsection{Problem Setup}

Let $\I \subseteq \R^d$ be a convex domain, and $N \in \N$. A \emph{symmetric} function 
$f: \I^{ N} \to \R$ is such that $f(x_1, \dots x_N) = f(x_{\pi(1)},\dots x_{\pi(N)})$ 
for any $x \in \I^{N}$ and any permutation $\pi \in \mathcal{S}_N$. In this work, we are interested in learning symmetric functions defined independently of $N$.  Let $\overline{\I}= \bigcup_{N=1}^\infty \I^N$, then $f: \overline{\I} \to \R$ is symmetric if $f$ restricted to  $\I^{ N}$ is symmetric for each $N \in \N$. Let $\mathcal{F}_{\mathrm{sym}}$ 
denote the space of symmetric functions defined on $\overline{\I}$. 
This setting is motivated by applications in statistical mechanics and particle physics, where $N$ is a sampling parameter. 

We focus on the realizable regression setting, where we observe a dataset $\{(\xb_i, f^*(\xb_i)) \in \overline{\I} \times \R \}_{i=1,\dots n}$ of $n$ samples from an unknown symmetric function $f^*$, and $\xb_i$ are drawn iid from a distribution $\mathcal{D}$ on $\overline{\I}$. The goal is to find a proper estimator $\hat{f} \in \mathcal{F}_{\mathrm{sym}}$ such that the population error
$\E_{\xb \sim \mathcal{D}} \ell(f^*(\xb), \hat{f}(\xb))$
is low, where $\ell$ is a convex loss. 

Following a standard Empirical Risk Minimisation setup \parencite{shalev2014understanding,bach2017breaking}, we will construct hypothesis classes $\mathcal{F} \subset \mathcal{F}_{\mathrm{sym}}$ endowed with a metric $\| f\|_\mathcal{F}$, and consider 
\begin{equation}
\label{eq:erm}
    \hat{f} \in \argmin_{f \in \mathcal{F}; \| f\|_{\mathcal{F}} \leq \delta} \frac{1}{n}\sum_{i=1}^n \ell( f^*(\xb_i), f(\xb_i) )~,
\end{equation}
where $\delta$ is a regularization parameter that is optimised using e.g. cross-validation.
We focus on the approximation and statistical aspects of this estimator for different choices of $\mathcal{F}$; solving the optimization problem (\ref{eq:erm}) is not the focus of the present work and will be briefly discussed in Section \ref{sec:conclusion}. 


\subsection{Symmetric Polynomials}
\label{sec:symmpoly}
A simplest way to approximate symmetric functions is with symmetric polynomials. Combining Weierstrass approximation theory with a symmetrization argument, it can be seen that assuming $d = 1$, any symmetric continuous function $f: \I^N \rightarrow \mathbb{R}$ can be uniformly approximated by symmetric polynomials (see~\textcite{yarotsky2018universal} for a proof).  There are several canonical bases over the ring of symmetric polynomials, but we will consider the one given by the power sum polynomials, given by $p_k(x) = \sum_{i=1}^N x_i^k$, with $x \in \I^N$.

\begin{theorem}[(2.12) in~\textcite{macdonald1998symmetric}]
    For any symmetric polynomial $f$ on $N$ inputs, there exists a polynomial $q$ such that $f(x) = q(p_1(x), \dots, p_N(x))$.
\end{theorem}

If $q$ is linear, this theorem suggests a simple predictor for symmetric functions across varying $N$.  If $x \in \I^M$, we can consider $x \mapsto \sum_{i=1}^N c_i \left(\frac{1}{M} p_i(x)\right) = \sum_{i=1}^N c_i \E_{y\sim\mu}(y^i)$ where $\mu = \frac{1}{M} \sum_{j=1}^M \delta_{x_j}$.  The truncated moments of the empirical distribution given by $x$ act as linear features, which yield an estimator over any input size $M$.  We will consider a generalization of this decomposition, by moving beyond the polynomial kernel to a general RKHS (see Section~\ref{sec:neural_functional_spaces}).





\subsection{Convex Shallow Neural Networks}\label{sec:convex_shallow}

By considering the limit of infinitely many neurons \textcite{bengio2006convex,rosset2007}, \textcite{bach2017breaking} introduces two norms on shallow neural representation of functions $\phi$ defined over $\R^d$.  For a constant $R \in \mathbb{R}$, a fixed probability measure $\kappa \in \mathcal{P}(\mathbb{S}^{d})$ with full support, a signed Radon measure $\nu \in \mathcal{M}(\mathbb{S}^{d})$, a density $p \in L_2(d \kappa)$,
and the notation that $\tilde{x} = [x, R]^T$, define:
\begin{align}
    \gamma_1(\phi) = \inf \left\{ \| \nu \|_{\mathrm{TV}}; \, \phi(x) = \int_{\mathbb{S}^{d}} \sigma( \langle w, \tilde{x} \rangle) 
    \nu(dw) \right\}~, \text{ and } \\
    \gamma_2(\phi) = \inf \left\{ \| p \|_{L_2(d \kappa)}; \, \phi(x) = \int_{\mathbb{S}^{d}} \sigma( \langle w, \tilde{x} \rangle) 
    p(w) \kappa(dw) \right\}~,
\end{align}


where $\| \nu \|_{\mathrm{TV}}:= \sup_{|g|\leq 1} \int g d\nu$ is the \emph{Total Variation} of $\nu$ and $\sigma_\alpha(t) = \max(0, t)^\alpha$ is the ReLU activation raised to the positive integer power $\alpha$.  These norms measure the minimal representation of $\phi$, using either a Radon measure $\nu$ over neuron weights, or a density $p$ over the fixed probability measure $\kappa$.  The norms induce function classes:
\begin{equation}
    \mathcal{F}_1  = \{\phi \in C_0(\I): \gamma_1(\phi) < \infty\} ~,\text{ and}~
    \mathcal{F}_2  = \{\phi \in C_0(\I): \gamma_2(\phi) < \infty\}~.
\end{equation}
We also assume that the input domain $\I$ is bounded with $\sup_{x\in\I} \|x\|_2 \leq R$.


These two functional spaces are fundamental for the theoretical study of shallow neural networks and capture two distinct regimes of overparametrisation: whereas the so-called \emph{lazy} or kernel regime corresponds to learning in the space $\mathcal{F}_2$ \parencite{chizat2018note,jacot2018neural}, which is in fact an RKHS with kernel given by $k(x,y)= \E_{w \sim \kappa} \left[\sigma_\alpha( \langle w, \tilde{x} \rangle) \sigma_\alpha( \langle w, \tilde{y} \rangle)\right]$ \parencite{bach2017breaking}  \footnote{Or a modified NTK kernel that also includes gradients with respect to first-layer weights \parencite{jacot2018neural} }, the mean-field regime captures learning in $\mathcal{F}_1$, which satisfies $\mathcal{F}_2 \subset \mathcal{F}_1$ from Jensen's inequality, and can efficiently approximate functions with hidden low-dimensional structure, as opposed to $\mathcal{F}_2$ \parencite{bach2017breaking}. 

Finally, one can leverage the fact that the kernel above is an expectation over features to define a finite-dimensional random feature kernel 
$k_m(x,y) = \frac{1}{m} \sum_{j=1}^m \sigma_\alpha( \langle w_j, \tilde{x} \rangle) \sigma_\alpha( \langle w_j, \tilde{y} \rangle)$ with $w_j \stackrel{i.i.d.}{\sim} \kappa$, which defines a (random) RKHS $\mathcal{F}_{2,m}$ converging to $\mathcal{F}_2$ as $m$ increases \parencite{bach2017equivalence, rahimi2008random}.  The empirical norm $\gamma_{2,m}$ can be defined similarly to $\gamma_2$, where the density $p$ is replaced by coefficients over the sampled basis functions $\sigma_\alpha(\langle w_j, \cdot \rangle)$.



\subsection{Symmetric Neural Networks}
\label{sec:symmNN}

A universal approximator for symmetric functions was proposed by ~\textcite{zaheer2017deep}, which proved that for any fixed $N$ and $f_N \in \mathcal{F}_{\mathrm{sym}}^N$ there must exist $\Phi: \I \rightarrow \mathbb{R}^L$ and $\rho: \mathbb{R}^L \rightarrow \mathbb{R}$ such that
\begin{equation}\label{eq:oursets}
    f_N(x) = \rho \left(\frac{1}{N} \sum_{n=1}^N \Phi(x_n) \right)~.
\end{equation}
However, universality is only proven for fixed $N$.  Given a symmetric function $f \in \F$ 
we might hope to learn $\rho$ and $\Phi$ such that this equation holds for all $N$.  Note that the fraction $\frac{1}{N}$ is not present in their formulation, but is necessary for generalization across $N$ to be feasible (as otherwise the effective domain of $\rho$ could grow arbitrarily large as $N \rightarrow \infty$).


Treating the input to $\rho$ as an average motivates moving from sets to measures as inputs, as proposed in \textcite{pevny2019approximation, de2019stochastic}.
Given $x \in \I^N$, let $\mu^{(N)} = \frac{1}{N} \sum_{i=1}^N \delta_{x_i}$ denote the 
empirical measure in the space $\mathcal{P}(\I)$ of probability measures over $\I$. 
Then (\ref{eq:oursets}) can be written as 
$f_N(x) = \rho\left( \int_{\I} \Phi(u) \mu^{(N)}(du) \right)~.$

\section{From Set to Measure Functions}\label{sec:set2measure}

\subsection{Neural Functional Spaces for Learning over Measures}\label{sec:neural_functional_spaces}

Equipped with the perspective of (\ref{eq:oursets}) acting on an empirical measure, we consider shallow neural networks that take probability measures as inputs, with test functions as weights. We discuss in Section \ref{sec:settomeasuresub} which functions defined over sets admit an extension to functions over measures. 

Let $\mathcal{A}$ be a subset of $C_0(\I)$, equipped with its Borel sigma algebra.  For $\mu \in \PI$, and a signed Radon measure $\chi \in \mathcal{M}(\mathcal{A})$, define ${f}: \PI \rightarrow \mathbb{R}$ as 
\begin{equation}
    {f}(\mu; \chi) = \int_{\mathcal{A}} \widetilde{\sigma}(\langle \phi, \mu \rangle) \chi(d\phi)~.
\end{equation}
where $\widetilde{\sigma}$ is again a scalar activation function, such as the ReLU, and $\langle \phi, \mu \rangle := \int_\I \phi(x) \mu(dx)$.  Crucially, the space of functions given by ${f}(\cdot; \chi)$ were proven to be dense in the space of real-valued continuous (in the weak topology) functions on $\PI$ in~\textcite{pevny2019approximation,de2019stochastic}, and so this network exhibits universality.

Keeping in mind the functional norms defined on test functions in Section~\ref{sec:convex_shallow}, we can introduce analogous norms for neural networks on measures. For a fixed probability measure $\tau \in \mathcal{P}(\mathcal{A})$, define
\begin{equation}
\label{eq:f11}
    \| f\|_{1, \mathcal{A}} = \inf \left\{ \| \chi \|_{\mathrm{TV}}; \,f(\mu) = \!\int_{\mathcal{A}} \widetilde{\sigma}( \langle \phi, \mu \rangle) \chi(d\phi) \right\}~,
\end{equation}
\begin{equation}
\label{eq:f22}
    \| f\|_{2, \mathcal{A}} = \inf \left\{ \| q \|_{L_2}; \, f(\mu) = \!\int_{\mathcal{A}} \widetilde{\sigma}( \langle \phi, \mu \rangle) q(\phi) \tau(d\phi) \right\}~,
\end{equation}
where we take the infima over Radon measures $\chi \in \mathcal{M}(\mathcal{A})$ and densities $q \in L_2(d\tau)$.  Analogously these norms also induce the respective function classes $\mathcal{G}_1(\mathcal{A}) = \{f: \|f\|_{1,\mathcal{A}} < \infty\}$, $\mathcal{G}_2(\mathcal{A}) = \{f: \|f\|_{2,\mathcal{A}} < \infty\}$.  The argument in Appendix A of~\textcite{bach2017breaking} implies $\mathcal{G}_2(\mathcal{A})$ is an RKHS, with associated kernel $k_\mathcal{G}(\mu, \mu') = \int_\mathcal{A} \widetilde{\sigma}( \langle \phi, \mu \rangle) \widetilde{\sigma}( \langle \phi, \mu' \rangle) \tau(d \phi)$.

Moving from vector-valued weights to function-valued weights presents an immediate issue.  The space $C_0(\I)$ is infinite-dimensional, and it is not obvious how to learn a measure $\chi$ over this entire space.  Moreover, our ultimate goal is to understand finite-width symmetric networks, so we would prefer the function-valued weights be efficiently calculable rather than pathological.  To that end, we choose the set of test functions $\mathcal{A}$ to be representable as regular neural networks.

Explicitly, using the function norms of Section~\ref{sec:convex_shallow}, we define

\begin{eqnarray*}
\label{eq:bobo}
    \mathcal{A}_{1,m} &:=& \left\{\phi; ~ \phi(x) = \sum_{j=1}^m \alpha_j \sigma(\langle w_j, \tilde{x} \rangle)~,\,\|w_j\|_2\leq 1, \| \alpha \|_1 \leq 1\right\} \nonumber, \\ 
    \mathcal{A}_{2,m} &:=& \left\{\phi \in \mathcal{F}_{2,m}:\, \gamma_{2,m}(\phi) \leq 1 \right\}~.
\end{eqnarray*}

 $\A_{1,m}$ thus contains functions in the unit ball of $\FF_1$ that can be expressed with $m$ neurons, and $\A_{2,m}$ contains functions in the (random) RKHS $\FF_{2,m}$ obtained by sampling $m$ neurons from $\kappa$. By definition $\A_{2,m} \subset \A_{1,m}$ for all $m$. 
 Representational power grows with $m$,
 and observe that the approximation rate in the unit ball of $\FF_1$ or $\FF_2$ is in $m^{-1/2}$, obtained for instance with Monte-Carlo estimators \parencite{bach2017breaking,ma2019barron}. Hence we can also consider the setting where $m = \infty$, with the notation $\mathcal{A}_{\{i,\infty\}} = \{\phi \in \FF_{i} : \gamma_{i}(\phi) \leq 1\}$.
Note also that there is no loss of generality in choosing the radius to be $1$, as by homogeneity of $\sigma$ any $\phi$ with $\gamma_i(\phi) < \infty$ can be scaled into its respective norm ball.

We now examine the combinations of $\mathcal{G}_i$ with $\mathcal{A}_i$:
\begin{itemize}
    \item $\mathcal{S}_{1,m} := \mathcal{G}_1(\mathcal{A}_{1,m})$; the measure $\chi$ is supported on test functions in $\mathcal{A}_{1,m}$. 
    \item $\mathcal{S}_{2,m} := \mathcal{G}_1(\mathcal{A}_{2,m})$; $\chi$ is supported on test functions in $\mathcal{A}_{2,m}$.
    \item $\mathcal{S}_{3,m} := \mathcal{G}_2(\mathcal{A}_{2,m})$; $\chi$ has a density with regards to $\tau$, which is supported on $\mathcal{A}_{2,m}$.
    \item The remaining class $\mathcal{G}_2(\mathcal{A}_{1,m})$ requires defining a probability measure $\tau$ over $\mathcal{A}_{1,m}$ that sufficiently spreads mass outside of any RKHS ball.  Due to the difficulty in defining this measure in finite setting, we omit this class.
\end{itemize}

Note that from Jensen's inequality and the inclusion $\mathcal{A}_{2,m} \subset \mathcal{A}_{1,m}$ for all $m$, we have the inclusions $\mathcal{S}_{3,m} \subset \mathcal{S}_{2,m} \subset \mathcal{S}_{1,m}$.  And $\mathcal{S}_{3,m}$ is clearly an RKHS, since it is a particular instantiation of $\mathcal{G}_2(\mathcal{A})$.
In the sequel we will drop the subscript $m$ and simply write $\A_i$ and $\S_i$.

These functional spaces provide an increasing level of adaptivity: while $\S_2$ is able to adapt by selecting `useful' test functions $\phi$, it is limited to smooth test functions that lie on the RKHS, whereas $\S_1$ is able to also adapt to more irregular test functions that themselves depend on low-dimensional structures from the input domain. 
We let $\|f\|_{\mathcal{S}_i}$ denote the associated norm, i.e. $\|f\|_{\mathcal{S}_1} := \|f\|_{1, \mathcal{A}_{1}}$.

\paragraph{Finite-Width Implementation:}


For any $m$, these classes admit a particularly simple interpretation when implemented in practice. On the one hand, the spaces of test functions are implemented as a single hidden-layer neural network of width $m$. On the other hand, the integral representations in (\ref{eq:f11}) and (\ref{eq:f22}) are instantiated by a finite-sum using ${m'}$ neurons, leading to the finite 
analogues of our function classes given in Table~\ref{tab:freeze}. Specifically, 
\begin{align*}
    f(\mu) = \frac{1}{m'}\! \sum_{j'=1}^{m'} b_{j'} \widetilde{\sigma}\! \left(\!\!\frac{1}{m} \sum_{j=1}^m c_{j',j} \int \sigma_\alpha( \langle w_{j',j}, \tilde{x} \rangle) \mu(dx)\!\! \right)
\end{align*}

One can verify \parencite{neyshabur_norm-based_2015} that the finite-width proxy for the  variation norm is given by
$$\| f \|_1 = \frac{1}{m'} \! \sum_{j'} |b_{j'}| \| \phi_{j'} \|_1 \leq \frac{1}{m m'} \!\sum_{j',j} |b_{j'}| |c_{j',j}| \| w_{j',j} \|~,$$
which in our case corresponds to the so-called path norm \parencite{neyshabur2014search}.  In particular, under the practical assumption that the test functions $\phi_{j'}$ are parameterized by two-layer networks with shared first layer, the weight vectors $w_{j',j}$ only depend on $j$ and this norm may be easily calculated as a matrix product of the network weights.  We can control this term by constraining the weights of the first two layers to obey our theoretical assumptions (of bounded weights and test functions in respective RKHS balls), and regularize the final network weights.  See Section~\ref{sec:experiments} and the Appendix for practical relaxations of the constraints.


\begin{table}
    \centering
    \begin{tabular}{p{5pt}c|cccccc}
         & & First Layer   &  Second Layer   & Third Layer   \\
         \hline
    & $\mathcal{S}_1$  & Trained  &  Trained &  Trained &  \\
     & $\mathcal{S}_2$  & Frozen &  Trained &  Trained & \\
      & $\mathcal{S}_3$  & Frozen &  Frozen &  Trained \\
    \end{tabular}
    \caption{Training for finite function approximation}
    \label{tab:freeze}
\end{table}

\subsection{Continuous Extension}
\label{sec:settomeasuresub}
In general, the functions we want to represent don't take in measures $\mu \in \PI$ as inputs. 
In this section, we want to understand when a function $f$ defined on the power set $f : \overline{\I} \to \R$ can be extended to a continuous map $\bar{f} : \PI \to \R$ in the weak topology, in the sense that for all $N \in \N$ and all $(x_1, \dots x_N)\in \I^N$, $\bar{f}\left(\frac{1}{N}\sum_{i=1}^N \delta_{x_i} \right) = f(x_1, \dots, x_N)$.

Observe that by construction $\bar{f}$ captures the permutation symmetry of the original $f$. 
Define the mapping $D: \overline{\I} \rightarrow \mathcal{P}(\I)$ by $D(x_1, \dots, x_N) = \frac{1}{N}\sum_{i=1}^N \delta_{x_i}$.  Let $\hat{\mathcal{P}}_N(\I) := D(\I^N)$ and $\hat{\mathcal{P}}(\I) = \bigcup_{N=1}^\infty \hat{\mathcal{P}}_N(\I)$, so that $\hat{\mathcal{P}}(\I)$ is the set of all finite discrete measures.  For $\mu \in \hat{\mathcal{P}}(\I)$, let $N(\mu)$ be the smallest dimension of a point in $D^{-1}(\mu)$, and let $x$ be this point (which is unique up to permutation).  Then define $\hat{f}: \hat{\mathcal{P}}(\I) \rightarrow \mathbb{R}$ such that $\hat{f}(\mu) = f_N(x)$.

We also write $W_1(\mu, \mu')$ as the Wasserstein 1-metric under the $\|\cdot\|_2$ norm~\parencite{villani2008optimal}.
The following proposition establishes a necessary and sufficient condition for continuous extension of $f$:  
\begin{proposition}\label{prop:extension}
    There exists a continuous extension $\bar{f}$ iff $\hat{f}$ is uniformly continuous with regard to the $W_1$ metric on its domain.
\end{proposition}

This result formalises the intuition that extending a symmetric function from sets to measures requires a minimal amount of regularity \emph{across} sizes. We next show examples of symmetric families that can be extended to $\PI$. 



\subsection{Examples of Eligible Symmetric Families}\label{sec:examples}

\paragraph{Moment-based Functions:}



Functions based on finite-range interactions across input elements admit continuous extensions.  For example, a function of singleton and pairwise interactions
$$f(x) = \rho\left(\frac{1}{N} \sum_{i=1}^N \phi_1(x_i), \frac{1}{N^2}\sum_{i_1,i_2=1}^N \phi_2(x_{i_1}, x_{i_2})\right)$$ is a special case of the continuous measure extension
$\bar{f}(\mu) = \rho \left( \langle \phi_1, \mu \rangle, \langle \phi_2, \mu \otimes \mu \rangle \right)$ when $\mu = D(x)$.

\paragraph{Ranking:}

Suppose that $\I \subseteq \R$. 
The max function $f_N(x) = \max_{i\leq N} x_i$ cannot be lifted to a function on measures due to discontinuity in the weak topology.  Specifically, consider $\mu = \delta_0$ and $\nu_{N} = \frac{N-1}{N} \delta_0 + \frac{1}{N} \delta_1$.  Then $\nu_N \rightharpoonup \mu$, but for $\hat{f}$ as in Proposition \ref{prop:extension}, $\hat{f}(\nu_N) = 1 \neq 0 = \hat{f}(\mu)$.

Nevertheless, we can define an extension on a smooth approximation via the softmax, namely
$g_N^\lambda(x) = \frac{1}{\lambda} \log \frac{1}{N} \sum_{i=1}^N \exp(\lambda x_i)$.  This formulation, which is the softmax up to an additive term, can clearly be lifted to a function on measures, with the bound $\|g_N^\lambda - f_N\|_\infty \leq \frac{\log N}{\lambda}$.  Although we cannot learn the max family across all $N$, we can approximate arbitrarily well for bounded $N$.


\paragraph{Counterexamples:}

Define the map $\Delta_k: \mathbb{R}^N \rightarrow \mathbb{R}^{kN}$ such that $\Delta_k(x)$ is a vector of $k$ copies of $x$.  Then a necessary condition for the function $\hat{f}$ introduced in Proposition~\ref{prop:extension} to be uniformly continuous is that $f_N(x) = f_{kN}(\Delta_k(x))$ for any $k$.  Intuitively, if $f_N$ can distinguish the input set beyond the amount of mass on each point, it cannot be lifted to measures.  This fact implies any continuous approximation to the family $f_N(x) = x_{[2]}$, the second largest value of $x$ will incur constant error.


\section{Approximation and Function Class Separation}
\label{sec:approx}


\subsection{Approximation of single `neurons'}
In the same spirit as the ``separations'' between $\FF_1$ and $\FF_2$, we characterise prototypical functions that belong to $\S_i$ but have poor approximation rates in $\S_{i+1}$ for $i=\{1,2\}$ in terms of the relevant parameters of the problem, the input dimensionality $d$ and the bandwidth parameter $m$. 
Such functions are given by single neurons in a spherical input regime (details for this setting are given in the Appendix).

For the remainder of this work, we consider $\widetilde{\sigma} = \sigma$ as the ReLU activation, and choose $\alpha = 2$ such that $\sigma_2(t) = \sigma(t)^2$ is the squared ReLU.


\begin{theorem}[informal]\label{thm:inclusion}
    Assume $m = \infty$.  For appropriate choices of the kernel base measures $\kappa$ and $\tau$, there exist $f_1$ with $\|f_1\|_{\S_1} \leq 1$ and $f_2$ with $\|f_2\|_{\S_2} \leq 1$ such that:
    \begin{align*}
        \inf_{\|f\|_{\S_3} \leq \delta} \|f - f_2\|_\infty & \gtrsim d^{-2} \delta^{-5/d}~,\\
        \inf_{ \|f \|_{\S_2} \leq \delta} \| f - f_1 \|_\infty & \gtrsim |d^{-11} - d^{-d/3}\delta|~.
    \end{align*}
\end{theorem}

The choice of the squared ReLU activation in the parameterization of the test functions is required in the proof separating $\S_1$ and $\S_2$.  This follows from some properties of spherical harmonic parity and the decomposition of signed measures into probability measures.

These separations use the infinity norm rather than an appropriate $L_2$ norm, and therefore hold in a weaker norm than separation between $\FF_1$ and $\FF_2$.  Nevertheless, these separations confirm that symmetric network expressiveness is \emph{graded} by the degree of non-linear learning.

Both results hold in the domain $m = \infty$, so from the concentration of the empirical kernel $k_m \rightarrow k$, with high probability these approximation lower bounds will still hold for sufficiently large $m$.  In finite-width implementations, however, $m$ may be sufficiently small that the random kernel more explicitly determines the expressiveness of $\S_{i,m}$.  We experimentally test the presence of these depth separations with finite $m$ in Section~\ref{sec:experiments}. 

\subsection{Approximation of variational symmetric function via Laplace method}

Consider any symmetric family $f_N(x) = \argmin_{t \in T} \langle \hat\mu_x, \phi_t \rangle$ where $\hat\mu_x$ is the empirical measure of $x$, ie, $\hat\mu_x = \frac{1}{N} \sum_i \delta_{x_i}$, $T$ is a Euclidean subset, and $t \mapsto \phi_t$ is measurable.  For example $T = \mathbb{R}$ and $\phi_t(x) = |t-x|$ yields $f_N$ as the median.

Although this function family isn't necessarily uniformly continuous in the weak topology, we highlight the option of a Laplace approximation.  Define $E_\mu(t) := \langle \mu, \phi_t \rangle $ and introduce the density $p_\beta(t) = \frac{1}{Z} e^{-\beta E_\mu(t)}$ where $Z = \int_T e^{-\beta E_\mu(t)} dt$ is the partition function.  Then consider the Gibbs approximation $g_\beta(\mu) := \E_{p_\beta}[t] = \frac{1}{Z} \int_T t e^{-\beta E_\mu(t)} dt$.


One can verify (e.g. \parencite{raginsky2017non}) that $g_\beta \to g$ pointwise at a rate $\sim  \frac{d \log(\beta + 1)}{\beta}$. As $g_\beta$ is continuous, by universality it can be represented in $\S_i$ for all $i=\{1,2,3\}$. 
An approximation of $g_\beta$ is given as a ratio of two shallow networks $g_\beta(\mu) = \frac{\int_T t \sigma_1(\langle \mu, \phi_t \rangle ) dt }{\int_T \sigma_1(\langle \mu, \phi_t \rangle ) dt}$,
with $\sigma_1(u) = e^{-\beta u}$. However, the approximation rates blow-up as $\beta \to \infty$ with an exponential dependency on the dimension of $T$. 

\section{Generalization and Concentration}
\label{sec:gen}


\subsection{Generalization Bounds}

Despite being a larger function class than $\FF_2$, the class $\FF_1$ enjoys a nice generalization bound \parencite{bach2017breaking}.  Crucially, this property is inherited when we lift to functions on measures, controlling the generalization of functions in $\S_1$:
\begin{proposition}\label{prop:rad}
    Assume for given $\delta$, for all $y$ the loss function $\ell(y, \cdot)$ is $G$-Lipschitz on $B_0(2R^2\delta)$, and $l(y, 0) \leq RG\delta$.  Then with probability at least $1-t$,

    $$\sup_{\|f\|_{\mathcal{S}_1} \leq \delta} \left| \E_{\mu \sim \mathcal{D}} \ell(f^*(\mu), f(\mu)) - \frac{1}{n} \sum_{i=1}^n \ell(f^*(\mu_i), f(\mu_i)) \right| \\
    \leq \frac{2RG\delta + 16R^4G\delta}{\sqrt{n}} + (4R^2G\delta + 2RG\delta)\sqrt{\frac{\log 1/t}{2n}}~.$$
\end{proposition}

This proposition demonstrates that learning in $\mathcal{S}_1$ is not cursed by the dimension of the underlying input space $\I$.  In other words, the main price for learning in $\S_1$ is not in generalization, despite the size of this class relative to $\S_2$ and $\S_3$.  In the absence of a lower bound on generalization error for the RKHS function classes, our experiments investigate the generalization of these models in practice.

Although $d$ and $N$ do not appear in this bound, these parameters nevertheless impact the generalization of our function classes $\S_i$.  The input dimension controls the separation of the classes according to Theorem~\ref{thm:inclusion}, and therefore larger $d$ weakens the generalization of $\S_2$ and $\S_3$; compare Figure~\ref{fig:plots} and Figure~\ref{fig:plots_extra} (in the Appendix) for how RKHS methods suffer in higher dimensions.  Whereas large $N$ and a natural choice of $\mathcal{D}$ make generalization for $\S_1$, and hence all three classes, nearly trivial, as discussed in section~\ref{sec:conc}.

\subsection{Concentration across Input Size}\label{sec:conc}


Consider the data distribution from which we sample, namely a measure from $ \mathcal{P}\left(\bar{\I} \right)$ to sample finite sets.  A natural way to draw data is to consider the following sampling procedure: given $\xi \in \mathcal{P}(\mathcal{P}(\I))$ and $\Omega \in \mathcal{P}(\mathbb{N})$, draw $\mu \sim \xi$ and $N \sim \Omega$, sample $N$ independent points $x_i \sim \mu$, and return $\{x_1, \dots, x_N\}$.  If $\xi$ is too peaked, this sampling process will concentrate very rapidly:

\begin{proposition}\label{prop:concentration}
    For $\xi = \delta_{\mu^*}$, then $\E \sup_{\|f\|_{\mathcal{S}_1} \leq \delta} \left|\frac{1}{n} \sum_{i=1}^n \epsilon_i f(\mu_i) \right| \lesssim \delta R^2( n^{-1/2} + \E_{N \sim \Omega}[N^{-1/d}])$.
\end{proposition}

Hence, the question of generalization across differently sized sets becomes trivial if $N$ is large and $d$ is small.  In our experiments, $N \approx d$, so we will nevertheless choose $\xi = \delta_{\mu}$ for some $\mu \in \mathcal{P}(\I)$.  We consider more exotic data distributions over measures in the experiments on robust mean estimation in Section~\ref{sec:robust_mean}.


\section{Experiments}\label{sec:experiments}

\begin{figure*}[ht]
\centering

\begin{subfigure}{.33\textwidth}
  \centering
  \includegraphics[width=1.\linewidth]{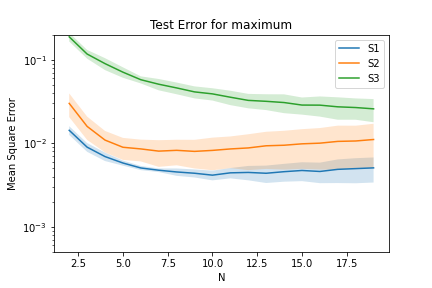}
\end{subfigure}%
\begin{subfigure}{.33\textwidth}
  \centering
  \includegraphics[width=1.\linewidth]{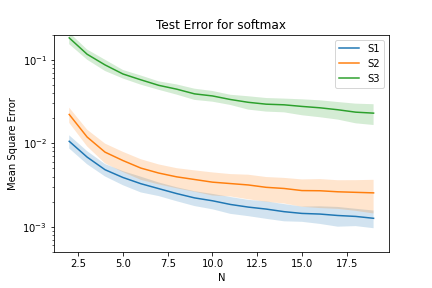}
\end{subfigure}%
\begin{subfigure}{.33\textwidth}
  \centering
  \includegraphics[width=1.\linewidth]{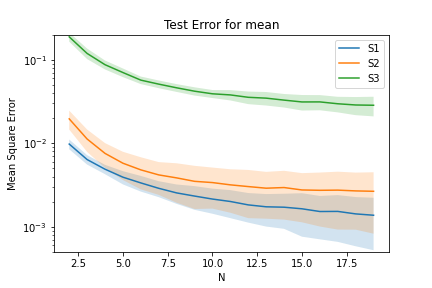}
\end{subfigure}%

\begin{subfigure}{.33\textwidth}
  \centering
  \includegraphics[width=1.\linewidth]{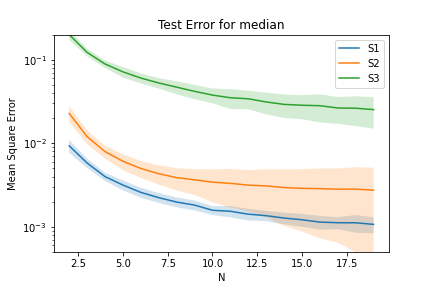}
\end{subfigure}%
\begin{subfigure}{.33\textwidth}
  \centering
  \includegraphics[width=1.\linewidth]{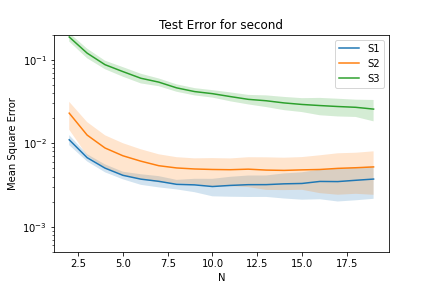}
\end{subfigure}%
\begin{subfigure}{.33\textwidth}
  \centering
  \includegraphics[width=1.\linewidth]{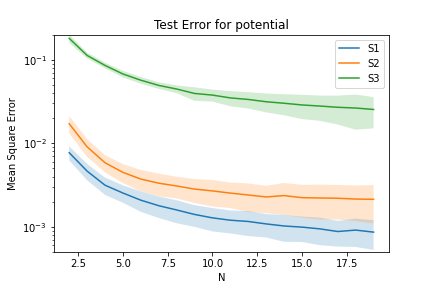}
\end{subfigure}%

\caption{\small{Test Error for $d = 10$ on the neural architectures of Section \ref{sec:neural_functional_spaces}}}
\vspace{-0.4cm}
\label{fig:plots}
\end{figure*}


\subsection{Symmetric Function Approximation}

We consider the task of learning several common symmetric functions (see Figure~\ref{fig:plots}).  Our aim is to practically understand the approximation bounds of Theorem~\ref{thm:inclusion}, as well as the generalization result of Proposition~\ref{prop:rad}.  Furthermore, by training and testing on sets of different sizes, we may consider how the models perform on out-of-distribution generalization across input size.

\paragraph{Experimental Setup:} We instantiate our three function classes in the finite network setting, as outlined in Table~\ref{tab:freeze}.  We use input dimension $d = 10$.  For the finite realization of $\S_1$, we use first hidden layer size $m = 100$ and second hidden layer size $h = 100$.  Crucially, after fixing the finite architecture representing $\S_1$, we scale up the width by 10 for the models with frozen weights.  That is, the first hidden layer in $\S_2$, and both hidden layers in $\S_3$, have width equal to 1000.  Increasing the width makes the $\S_2$ and $\S_3$ models strictly more powerful, and this setup allows us to inspect whether a larger number of random kernel features can compensate for a smaller, trained weight in approximation.  For each model, we use its associated functional norm for regularization.

Each network is trained on a batch of 100 input sets.  For our data distribution we consider the base domain $\I = [-3, 3]^d$, and the distribution over input measures $\xi$ places all its mass on the uniform measure $U([-3,3]^d)$.  We choose to train with $N = 4$, i.e. all networks train on input sets of size 4, and test on sets of varying size.  From the results we can measure out-of-distribution generalization of finite sets.

The one-dimensional symmetric functions are defined on sets of vectors by first applying inverse norms, i.e. $f_N(x) = \max_{1\leq i \leq N} \|x_i\|_2^{-1}$.  The potential function calculates the normalized gravitational potential among equal masses, i.e. $f_N(x) = \frac{2}{N(N-1)} \sum_{i < j} \frac{1}{\|x_i - x_j\|_2}$.  The planted neuron and smooth neuron are given as single-neuron networks, where following from the proof of Theorem~\ref{thm:inclusion}, the planted neuron weight initialization is distinct from the model weight initialization. Further details are given in the Appendix.

We additionally consider an applied experiment on a variant of MNIST to observe how the finite-width implementations perform on real-world data, by first mapping images to point clouds.  Due to space limitations, details and results are given in the Appendix.

\begin{figure*}[h]
\centering

\begin{subfigure}{.32\textwidth}
  \centering
  \includegraphics[width=1.\linewidth]{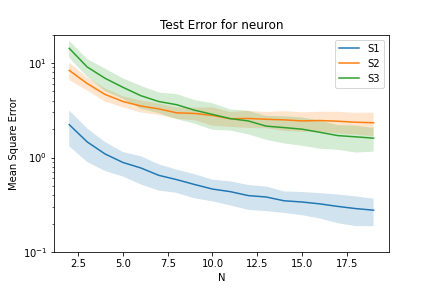}
\end{subfigure}%
\begin{subfigure}{.32\textwidth}
  \centering
  \includegraphics[width=1.\linewidth]{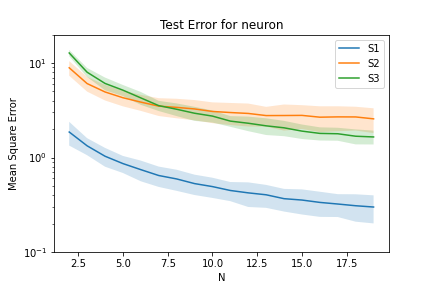}
\end{subfigure}%

\begin{subfigure}{.32\textwidth}
  \centering
  \includegraphics[width=1.\linewidth]{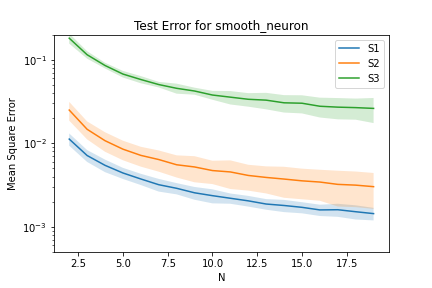}
\end{subfigure}%
\begin{subfigure}{.32\textwidth}
  \centering
  \includegraphics[width=1.\linewidth]{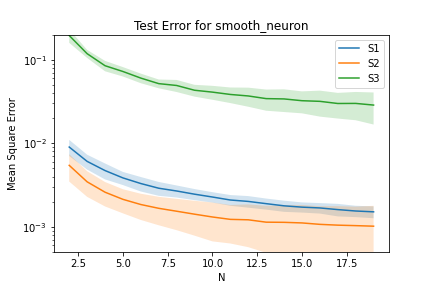}
\end{subfigure}%
\caption{\small{Planted neurons for $m = 100$ (left two) and $m = 200$ (right two).  The smooth neuron has weights sampled consistently with $\mathcal{F}_2$ while the regular neuron has weights sampled distinctly from the network initialization.}}
\label{fig:neurons}
\end{figure*}


\paragraph{Discussion:}
We observe in Figure~\ref{fig:plots} that $\mathcal{S}_3$ performs substantially worse in several cases, consistent with this function class being the smallest of those considered.  The classes $\mathcal{S}_2$ and $\mathcal{S}_1$ are competitive for some functions, although we observe a trend where $\S_1$ still has better generalization performance.  Therefore, the larger number of random kernel features doesn't compensate for training a smaller weight matrix in $\S_1$, empirically confirming Theorem~\ref{thm:inclusion}.

The test error on sets of larger size than the training data corroborates the conclusion of Proposition~\ref{prop:extension}.  The second-largest-element function generalizes extremely poorly, consistent with the observation in Section~\ref{sec:examples} that this function family cannot be approximated without constant error.  In particular, all function classes more effectively generalize across different $N$ on the softmax than the max, seeing as the latter lacks uniform continuity in measure space.

The other essential takeaway is the performance of the three models on the planted neurons in Figure~\ref{fig:neurons}.  By using a distinct weight initialization for the neuron, its first layer will have very little mass under $\kappa$, and its first two layers will have little mass under $\tau$, and therefore random features will not suffice to approximate this neuron.  This is true even with the scaling of $\S_2$ and $\S_3$ to enable more random kernel features, reiterating that these single neuron functions realize a meaningful separation between the classes.  We observe a more similar performance of $\S_1$ and $\S_2$ on the smooth\_neuron, as this function is chosen to be exactly representable with the random kernel features sampled by $\S_2$.  According to the function class inclusion it is still representable by $\S_1$, but from Theorem~\ref{thm:inclusion} not efficiently representable by $\S_3$, which is consistent with the results.

\begin{wrapfigure}{r}{0.5\textwidth}
    \includegraphics[width=0.9\linewidth]{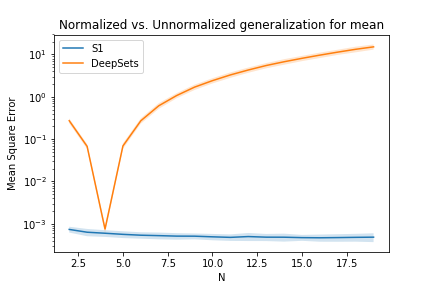}
    \caption{Test error for $\mathcal{S}_1$ versus unnormalized DeepSets architecture.}
    \label{fig:deepsets}
\end{wrapfigure}


On increasing $m$, the standard deviations of $\S_2$ and $\S_3$ shrink with more random kernel features, but $\mathcal{S}_1$ still achieves the best approximation on the neuron.  For the smooth neuron, $\S_1$ and $\mathcal{S}_2$ perform comparably, but $\S_3$ performs worse even for larger $m$.   In Figure~\ref{fig:deepsets} we confirm the need for taking averages rather than sums in the DeepSets architecture, as the unnormalized model cannot generalize outside of the value of $N = 4$ where it was trained.

\subsection{Robust Mean Estimation}\label{sec:robust_mean}

Symmetric functions naturally arise in the context of empirical estimators.  We consider specifically the task of robust mean estimation~\cite{diakonikolas2017being}, where one seeks to estimate $\E_{X \sim P}[X]$ given samples drawn from the mixture distribution $(1-\epsilon) P + \epsilon Q$.  For simplicity, we consider an oblivious contamination model where the true distribution $P$ and the noise distribution $Q$ have similar mean vectors.  Explicitly, each input set is derived as follows: we sample $m \sim \mathcal{N}(0, \sigma_m^2 I)$, $m' \sim \mathcal{N}(m, \sigma_{m'}^2 I)$, and define $P = \mathcal{N}(m, \sigma_P^2 I)$ and $Q = \mathcal{N}(m', \sigma_Q^2 I)$.  Then each input sets consist of $N$ samples $(X_1, \dots, X_N)$ where $X_i \overset{iid}{\sim} (1-\epsilon) P + \epsilon Q$.  Note that each input set is a corrupted sample with a different true mean vector $m$.

\begin{table*}[ht!]
    \centering
    \begin{tabular}{p{5pt}c|c|c|c|c}
         & &  $N = 10$ &  $\color{blue}N = 20$ & $N = 30$ & $N = 40$   \\
         \hline
    & $\mathcal{S}_1$ & $0.335 \pm 0.153$ & $\color{blue}0.131 \pm 0.018$ & $0.091 \pm 0.011$ & $0.076 \pm 0.011$ \\
     & $\mathcal{S}_2$  &  $0.342 \pm 0.153$ & $\color{blue}0.137 \pm 0.019$ & $0.098 \pm 0.012$ & $0.082 \pm 0.011$ \\
      & $\mathcal{S}_3$  & $0.361 \pm 0.162$ & $\color{blue}0.144 \pm 0.020$ & $0.103 \pm 0.013$ & $0.087 \pm 0.013$ \\
      \hline
      \hline
    & Sample Mean  & $0.385 \pm 0.172$ & $\color{blue}0.153 \pm 0.068$ & $0.093 \pm 0.042$ & $0.096 \pm 0.043$\\
     & Geometric Median  &  $0.321 \pm 0.144$ & $\color{blue}0.138 \pm 0.062$ & $0.087 \pm 0.039$ & $0.077 \pm 0.034$ \\
      & Adversarial Estimator  & $0.612 \pm 0.495$ & $\color{blue}0.469 \pm 0.550$ & $0.417 \pm 0.549$ & $0.420 \pm 0.564$\\
    \end{tabular}
    \caption{Mean squared test error for robust mean estimation among the finite model instantiations and baselines.}  
    \label{tab:robust}
\end{table*}

\paragraph{Experimental Setup:}

The network architecture is the same as above, with $d=10$. We use $\sigma_m = 1$, $\sigma_{m'} = 2$, $\sigma_{P} = \sigma_{Q} = 1.5$, and $\epsilon = 0.2$.  All networks train on sets of size $N = 20$, and test on sets of varying size, with mean squared error as the objective.  As baselines we consider the naive sample mean, the geometric median, and the adversarially robust mean estimator proposed in~\cite{diakonikolas2017being}.  The results are given in Table~\ref{tab:robust}.

\paragraph{Discussion:}


Although the variance is quite high due to the sampling procedure, performance in this setting confirms that robust mean estimation also realizes the class separation, and that for this simple corruption model learning is competitive and in some cases superior to fixed estimators.  In particular, the advantage of $\S_1$ over the baselines is most clear for $N = 20$, the setting where it was trained.  Although the dependence of the fixed estimators on $\sigma_P$ and $\sigma_Q$ vanishes as $N \rightarrow \infty$, the dependence on these parameters is non-negligible in the regime where $N$ is small, and therefore the robust mean may not generalize in the sense of Proposition~\ref{prop:extension}.  We explore training on different $N$ sizes further in the Appendix.  The poor performance of the adversarial estimator can mainly be attributed to the fact that the number of samples is considerably smaller than the setting studied in~\cite{diakonikolas2017being}, weakening the concentration of the empirical covariance matrix on which this estimator relies.



\section{Conclusion}
\label{sec:conclusion}

In this work, we have analyzed learning and generalization of symmetric functions through the lens of neural networks defined over probability measures, which formalizes the learning of symmetric function families across varying input size. Our experimental data confirms the theoretical insights distinguishing tiers of non-linear learning, and suggests that symmetries in the input might be a natural device to study the functional spaces defined by deeper neural networks. 
Specifically, and by focusing on shallow architectures, our analysis extends the fundamental separation between 
adaptive and non-adaptive neural networks from \textcite{bach2017breaking} to symmetric functions, leading to a hierarchy of functional spaces $\S_3 \subset \S_2 \subset \S_1$, in which nonlinear learning is added into the parametrization of the network weights ($\S_2$), and into the parametrization of test functions ($\S_1$) respectively.


A crucial aspect we have not addressed, though, is the computational cost of learning in $\S_1$ through gradient-based algorithms. An important direction of future work is to build on recent advances in mean-field theory for learning shallow neural networks \parencite{chizat2020implicit, ma2019barron, ma2020quenching,de2020sparsity}.

\paragraph{Acknowledgements:}
We thank Raghav Singhal for helpful discussions regarding the proof of Theorem~\ref{thm:inclusion}.  This work has been partially supported by the Alfred P. Sloan Foundation, NSF RI-1816753, NSF CAREER CIF-1845360, and NSF CCF-1814524.

\printbibliography

@article{mhaskar1997neural,
  title={Neural networks for functional approximation and system identification},
  author={Mhaskar, Hrushikesh Narhar and Hahm, Nahmwoo},
  journal={Neural Computation},
  volume={9},
  number={1},
  pages={143--159},
  year={1997},
  publisher={MIT Press}
}

@article{stinchcombe1999neural,
  title={Neural network approximation of continuous functionals and continuous functions on compactifications},
  author={Stinchcombe, Maxwell B},
  journal={Neural Networks},
  volume={12},
  number={3},
  pages={467--477},
  year={1999},
  publisher={Elsevier}
}

@article{rossi2005functional,
  title={Functional multi-layer perceptron: a non-linear tool for functional data analysis},
  author={Rossi, Fabrice and Conan-Guez, Brieuc},
  journal={Neural networks},
  volume={18},
  number={1},
  pages={45--60},
  year={2005},
  publisher={Elsevier}
}

@article{sandberg1996network,
  title={Network approximation of input-output maps and functionals},
  author={Sandberg, Irwin W and Xu, Lilian},
  journal={Circuits, Systems and Signal Processing},
  volume={15},
  number={6},
  pages={711--725},
  year={1996},
  publisher={Springer}
}

@inproceedings{rosset2007,
  title={$\ell_1$ regularization in infinite dimensional feature spaces},
  author={Rosset, Saharon and Swirszcz, Grzegorz and Srebro, Nathan and Zhu, Ji},
  booktitle={International Conference on Computational Learning Theory},
  pages={544--558},
  year={2007},
  organization={Springer}
}

@inproceedings{bengio2006convex,
  title={Convex neural networks},
  author={Bengio, Yoshua and Roux, Nicolas L and Vincent, Pascal and Delalleau, Olivier and Marcotte, Patrice},
  booktitle={Advances in neural information processing systems},
  pages={123--130},
  year={2006}
}

@article{de2020sparsity,
  title={On Sparsity in Overparametrised Shallow ReLU Networks},
  author={de Dios, Jaume and Bruna, Joan},
  journal={arXiv preprint arXiv:2006.10225},
  year={2020}
}

@article{ma2020quenching,
  title={The Quenching-Activation Behavior of the Gradient Descent Dynamics for Two-layer Neural Network Models},
  author={Ma, Chao and Wu, Lei and others},
  journal={arXiv preprint arXiv:2006.14450},
  year={2020}
}

@article{raginsky2017non,
  title={Non-convex learning via stochastic gradient Langevin dynamics: a nonasymptotic analysis},
  author={Raginsky, Maxim and Rakhlin, Alexander and Telgarsky, Matus},
  journal={arXiv preprint arXiv:1702.03849},
  year={2017}
}

@inproceedings{de2019stochastic,
  title={Stochastic deep networks},
  author={De Bie, Gwendoline and Peyr{\'e}, Gabriel and Cuturi, Marco},
  booktitle={International Conference on Machine Learning},
  pages={1556--1565},
  year={2019}
}

@article{neyshabur_norm-based_2015,
	title = {Norm-{Based} {Capacity} {Control} in {Neural} {Networks}},
	url = {http://arxiv.org/abs/1503.00036},
	urldate = {2019-12-18},
	journal = {arXiv:1503.00036 [cs, stat]},
	author = {Neyshabur, Behnam and Tomioka, Ryota and Srebro, Nathan},
	month = apr,
	year = {2015},
	note = {arXiv: 1503.00036}
}

@inproceedings{diakonikolas2017being,
  title={Being robust (in high dimensions) can be practical},
  author={Diakonikolas, Ilias and Kamath, Gautam and Kane, Daniel M and Li, Jerry and Moitra, Ankur and Stewart, Alistair},
  booktitle={International Conference on Machine Learning},
  pages={999--1008},
  year={2017},
  organization={PMLR}
}

@article{neyshabur2014search,
  title={In search of the real inductive bias: On the role of implicit regularization in deep learning},
  author={Neyshabur, Behnam and Tomioka, Ryota and Srebro, Nathan},
  journal={arXiv preprint arXiv:1412.6614},
  year={2014}
}

@article{scarselli2008graph,
	Author = {Scarselli, Franco and Gori, Marco and Tsoi, Ah Chung and Hagenbuchner, Markus and Monfardini, Gabriele},
	Journal = {IEEE Transactions on Neural Networks},
	Number = {1},
	Pages = {61--80},
	Publisher = {IEEE},
	Title = {The graph neural network model},
	Volume = {20},
	Year = {2008}}

@article{ma2019barron,
	Author = {Ma, Chao and Wu, Lei and E, Weinan},
	Journal = {arXiv preprint arXiv:1906.08039},
	Title = {Barron spaces and the compositional function spaces for neural network models},
	Year = {2019}}

@article{bach2017breaking,
	Author = {Bach, Francis},
	Journal = {The Journal of Machine Learning Research},
	Number = {1},
	Pages = {629--681},
	Publisher = {JMLR. org},
	Title = {Breaking the curse of dimensionality with convex neural networks},
	Volume = {18},
	Year = {2017}}

@inproceedings{rahimi2008random,
	Author = {Rahimi, Ali and Recht, Benjamin},
	Booktitle = {Advances in neural information processing systems},
	Pages = {1177--1184},
	Title = {Random features for large-scale kernel machines},
	Year = {2008}}

@article{chizat2018note,
	Author = {Chizat, Lenaic and Bach, Francis},
	Journal = {arXiv preprint arXiv:1812.07956},
	Title = {A note on lazy training in supervised differentiable programming},
	Year = {2018}}

@inproceedings{jacot2018neural,
	Author = {Jacot, Arthur and Gabriel, Franck and Hongler, Cl{\'e}ment},
	Booktitle = {Advances in neural information processing systems},
	Pages = {8571--8580},
	Title = {Neural tangent kernel: Convergence and generalization in neural networks},
	Year = {2018}}

@article{kingma2014adam,
	Author = {Kingma, Diederik P and Ba, Jimmy},
	Journal = {arXiv preprint arXiv:1412.6980},
	Title = {Adam: A method for stochastic optimization},
	Year = {2014}}

@book{efthimiou2014spherical,
  title={Spherical harmonics in p dimensions},
  author={Efthimiou, Costas and Frye, Christopher},
  year={2014},
  publisher={World Scientific}
}

@article{chizat2020implicit,
	Author = {Chizat, Lenaic and Bach, Francis},
	Journal = {arXiv preprint arXiv:2002.04486},
	Title = {Implicit Bias of Gradient Descent for Wide Two-layer Neural Networks Trained with the Logistic Loss},
	Year = {2020}}

@article{bach2017equivalence,
  title={On the equivalence between kernel quadrature rules and random feature expansions},
  author={Bach, Francis},
  journal={The Journal of Machine Learning Research},
  volume={18},
  number={1},
  pages={714--751},
  year={2017},
  publisher={JMLR. org}
}

@article{lecun1998gradient,
  title={Gradient-based learning applied to document recognition},
  author={LeCun, Yann and Bottou, L{\'e}on and Bengio, Yoshua and Haffner, Patrick},
  journal={Proceedings of the IEEE},
  volume={86},
  number={11},
  pages={2278--2324},
  year={1998},
  publisher={Ieee}
}

@book{shalev2014understanding,
  title={Understanding machine learning: From theory to algorithms},
  author={Shalev-Shwartz, Shai and Ben-David, Shai},
  year={2014},
  publisher={Cambridge university press}
}

@article{yarotsky2018universal,
  title={Universal approximations of invariant maps by neural networks},
  author={Yarotsky, Dmitry},
  journal={arXiv preprint arXiv:1804.10306},
  year={2018}
}

@inproceedings{wagstaff2019limitations,
  title={On the Limitations of Representing Functions on Sets},
  author={Wagstaff, Edward and Fuchs, Fabian and Engelcke, Martin and Posner, Ingmar and Osborne, Michael A},
  booktitle={International Conference on Machine Learning},
  pages={6487--6494},
  year={2019}
}

@inproceedings{maron2019universality,
  title={On the Universality of Invariant Networks},
  author={Maron, Haggai and Fetaya, Ethan and Segol, Nimrod and Lipman, Yaron},
  booktitle={International Conference on Machine Learning},
  pages={4363--4371},
  year={2019}
}

@inproceedings{zaheer2017deep,
  title={Deep sets},
  author={Zaheer, Manzil and Kottur, Satwik and Ravanbakhsh, Siamak and Poczos, Barnabas and Salakhutdinov, Russ R and Smola, Alexander J},
  booktitle={Advances in neural information processing systems},
  pages={3391--3401},
  year={2017}
}

@inproceedings{qi2017pointnet,
  title={Pointnet: Deep learning on point sets for 3d classification and segmentation},
  author={Qi, Charles R and Su, Hao and Mo, Kaichun and Guibas, Leonidas J},
  booktitle={Proceedings of the IEEE conference on computer vision and pattern recognition},
  pages={652--660},
  year={2017}
}

@article{pevny2019approximation,
  title={Approximation capability of neural networks on spaces of probability measures and tree-structured domains},
  author={Pevny, Tomas and Kovarik, Vojtech},
  journal={arXiv preprint arXiv:1906.00764},
  year={2019}
}

@article{fournier2015rate,
  title={On the rate of convergence in Wasserstein distance of the empirical measure},
  author={Fournier, Nicolas and Guillin, Arnaud},
  journal={Probability Theory and Related Fields},
  volume={162},
  number={3-4},
  pages={707--738},
  year={2015},
  publisher={Springer}
}

@inproceedings{he2015delving,
  title={Delving deep into rectifiers: Surpassing human-level performance on imagenet classification},
  author={He, Kaiming and Zhang, Xiangyu and Ren, Shaoqing and Sun, Jian},
  booktitle={Proceedings of the IEEE international conference on computer vision},
  pages={1026--1034},
  year={2015}
}

@book{macdonald1998symmetric,
  title={Symmetric functions and Hall polynomials},
  author={Macdonald, Ian Grant},
  year={1998},
  publisher={Oxford university press}
}

@book{villani2008optimal,
  title={Optimal transport: old and new},
  author={Villani, C{\'e}dric},
  volume={338},
  year={2008},
  publisher={Springer Science \& Business Media}
}

@inproceedings{frogner2015learning,
  title={Learning with a Wasserstein loss},
  author={Frogner, Charlie and Zhang, Chiyuan and Mobahi, Hossein and Araya, Mauricio and Poggio, Tomaso A},
  booktitle={Advances in neural information processing systems},
  pages={2053--2061},
  year={2015}
}

@inproceedings{hashimoto2016learning,
  title={Learning population-level diffusions with generative RNNs},
  author={Hashimoto, Tatsunori and Gifford, David and Jaakkola, Tommi},
  booktitle={International Conference on Machine Learning},
  pages={2417--2426},
  year={2016}
}

\newpage
\appendix
\section{Omitted Proofs}

\begin{table}
\centering
    \caption{Summary of Notation}
    \label{tab:notation}
        \begin{tabular}{lp{11cm}}
             \toprule
             \textbf{Notation} &  \textbf{Definition}\\
             \midrule
            $\I$ & Input domain, subset of $\mathbb{R}^d$\\
            $\overline{\I}$ & $\bigcup_{N=1}^\infty \I^N$\\
            $\A_i$ & Classes of test functions $\I \rightarrow \mathbb{R}$\\
            $\gamma_i$ & Test function norm\\
            $\S_i$ & Class of functions mapping $\mathcal{P}(\I) \rightarrow \mathbb{R}$\\
            $\|\cdot\|_{\S_i}$ & Measure network norm\\
            $D$ & Map from vectors to empirical measures s.t. $D(x_1 \dots x_n) = \sum_{i=1}^n \delta_{x_i}$\\
            $\hat{\mathcal{P}}(\I)$ & $\bigcup_{N=1}^\infty D(\I^N)$\\
            $\kappa$ & Fixed probability measure over $\mathbb{S}^d$ in first layer\\
            $\nu$ & Signed measure over $\mathbb{S}^d$ in first layer\\
            $\tau$ & Fixed probability measure over $\A_i$ in second layer\\
            $\chi$ & Signed measure over $\A_i$ in second layer\\
            $Y_{k,j}$ & Orthogonal basis polynomial of degree $k$ and index $j$ on $\mathbb{S}^d$\\
            $P_k$ & Legendre polynomial of degree $k$\\
            $g_k$ & the $k$th spherical harmonic of a function $g: \mathbb{S}^d \rightarrow \mathbb{R}$\\
            \bottomrule
        \end{tabular}
\end{table}

\subsection{Proof of Proposition \ref{prop:extension}}

\begin{proof}

We remind our notation.  Given $f: \I \rightarrow \mathbb{R}$, the empirical extension $\hat{f}: \hat{\mathcal{P}}(\I) \rightarrow \mathbb{R}$ is defined as $\hat{f}(\mu) := f(x_\mu)$ where $x_\mu \in D^{-1}(\mu)$ and $\|x_\mu\|_0 = \min_{x \in D^{-1}(\mu)} \|x\|_0$.  And for $\bar{f}: \mathcal{P}(\overline{\I}) \rightarrow \mathbb{R}$, we say this is a continuous extension of $f$ if $\bar{f}$ is continuous in under the Wasserstein metric, and $f(x) = \bar{f}(D(x))$ for every real, finite-dimensional vector $x$.

For the forward implication, if $\bar{f}$ is a continuous extension, then clearly $\bar{f} = \hat{f}$ restricted to $\hat{\mathcal{P}}(\I)$.

Furthermore, continuity of $\bar{f}$ and compactness of $\mathcal{P}(\I)$ implies $\bar{f}$ is uniformly continuous, and therefore $\hat{f}$ is as well.

For the backward implication, we introduce $\hat{f}_\epsilon(\mu) = \sup_{\nu \in B_\epsilon(\mu) \cap \hat{\mathcal{P}}(\I)} \hat{f}(\nu)$ where the ball $B_\epsilon(\mu)$ is defined with the Wasserstein metric.  Note that $\hat{f}_\epsilon$ is defined over arbitrary probability measures, not just discrete measures.  Now, we introduce $\bar{f}(\mu) = \inf_{\epsilon > 0} \hat{f}_\epsilon(\mu)$, where density of the discrete measures and uniform continuity of $\hat{f}$ guarantees that $\bar{f}$ is well-defined and finite.

Uniform continuity implies if $\mu \in \hat{\mathcal{P}}(\I)$ then $\bar{f}(\mu) = \hat{f}(\mu)$.  Consider any $y \in \I^M$ such that $\mu = D(y)$, and define a sequence of vectors $y^i = (z_i, y_2, \dots, y_M)$ where $z_i \rightarrow y_1$ and all $z_i$ are distinct from elements of $y$.  Every point $y^i \in \I^M$ has a unique coordinate and therefore $\hat{f}(D(y^i)) = f_M(y^i)$.  Because $D(y^i) \rightharpoonup D(y)$, continuity implies $\hat{f}(D(y)) = f_M(y)$.  Thus, for any $y \in \I^M$, $\bar{f}(D(y)) = f_M(y)$, which implies $\bar{f}$ is an extension.

Now, suppose we have an arbitrary convergent sequence of probability measures $\mu_n \rightharpoonup \mu$.  By the density of discrete measures, we can define sequences $\mu_n^m \rightharpoonup \mu_n$ where $\mu_n^m \in \hat{\mathcal{P}}(\I)$.  In particular, we may choose these sequences such that for all $n$, $W_1(\mu_n^m, \mu_n) \leq \frac{1}{m}$.  Then for any $\epsilon > 0$,
\begin{align*}
    |\bar{f}(\mu) - \bar{f}(\mu_n)| \leq |\bar{f}(\mu) - \hat{f}_\epsilon(\mu)| + |\hat{f}_\epsilon(\mu) - \hat{f}(\mu_n^n)| + |\hat{f}(\mu_n^n) - \hat{f}_\epsilon(\mu_n)| + |\hat{f}_\epsilon(\mu_n) - \bar{f}(\mu_n)|~.
\end{align*}
Consider the simultaneous limit as $n \rightarrow \infty$ and $\epsilon \rightarrow 0$.  On the RHS, the first term vanishes by definition, and the fourth by uniform continuity.  For any $\nu \in B_\epsilon(\mu) \cap \hat{\mathcal{P}}(\I)$, $W_1(\nu, \mu_n^n) \leq W_1(\nu, \mu) + W_1(\mu, \mu_n) + W_1(\mu_n, \mu_n^n) \rightarrow 0$ in the limit.  So the second term vanishes as well by uniform continuity of $\hat{f}$.  Similarly, for any $\nu \in B_\epsilon(\mu_n) \cap \hat{\mathcal{P}}(\I)$, $W_1(\nu, \mu_n^n) \leq W_1(\nu, \mu_n) + W_1(\mu_n, \mu_n^n) \rightarrow 0$, and the third term vanishes by uniform continuity.  This proves continuity of $\bar{f}$.

\end{proof}

\subsection{Proof of Proposition \ref{prop:rad}}

\begin{proof} We can decompose the generalization error:

\begin{align*}
    \E & \sup_{\|f\|_{\mathcal{S}_1} \leq \delta} \left| \E_{\mu \sim \mathcal{D}} \ell(f^*(\mu), f(\mu)) - \frac{1}{n} \sum_{i=1}^n \ell(f^*(\mu_i), f(\mu_i)) \right| \\
    & \leq 2\E \sup_{\|f\|_{\mathcal{S}_1} \leq \delta} \left|\frac{1}{n} \sum_{i=1}^n \epsilon_i \ell(f^*(\mu_i), f(\mu_i)) \right| \\
    & \leq 2\E \sup_{\|f\|_{\mathcal{S}_1} \leq \delta} \left|\frac{1}{n} \sum_{i=1}^n \epsilon_i \ell(f^*(\mu_i), 0) \right| + 2\E \sup_{\|f\|_{\mathcal{S}_1} \leq \delta} \left|\frac{1}{n} \sum_{i=1}^n \epsilon_i (\ell(f^*(\mu_i), 0) - \ell(f^*(\mu_i), f(\mu_i))) \right| \\
    & \leq \frac{2RG\delta}{\sqrt{n}} + 4R^2G\E \sup_{\|f\|_{\mathcal{S}_1} \leq \delta} \left|\frac{1}{n} \sum_{i=1}^n \epsilon_i f(\mu_i) \right| ~,
\end{align*}
where the second step uses symmetrization through the Rademacher random variable $\epsilon$, and the fourth is by assumption on the loss function $\ell$, from the fact that $\|f\|_{\mathcal{S}_1} \leq \delta$ implies $\|f\|_\infty \leq 2R^2\delta$.  We decompose the Rademacher complexity (removing the absolute value by symmetry): 
\begin{align*}\label{eq:rademacher_start}
    \E \left[\sup_{\|f\|_{\mathcal{S}_1} \leq \delta} \frac{1}{n} \sum_{i=1}^n \epsilon_i f(\mu_i) \right]
    & = \E \left[\sup_{\substack{\chi \in \mathcal{M}(\mathcal{A}) \\ \|\chi\|_{TV} \leq \delta}}  \frac{1}{n} \sum_{i=1}^n \epsilon_i \int \sigma(\langle \phi, \mu_i \rangle) \chi(d \phi)  \right] \\
    & = \delta \E \left[\sup_{\gamma_1(\phi) \leq 1}  \frac{1}{n} \sum_{i=1}^n \epsilon_i \sigma(\langle \phi, \mu_i \rangle)  \right] \\
    & \leq \delta \E \left[\sup_{\gamma_1(\phi) \leq 1}  \frac{1}{n} \sum_{i=1}^n \epsilon_i \langle \phi, \mu_i \rangle  \right]~,
\end{align*}
where the last step uses the contraction lemma and that $\sigma$ is 1-Lipschitz.

Now, using the neural network representation of $\phi$:
\begin{align*}
    \E \left[\sup_{\|f\|_{\mathcal{S}_1} \leq \delta} \frac{1}{n} \sum_{i=1}^n \epsilon_i f(\mu_i) \right] & \leq \delta \E \left[\sup_{\|\nu\|_{TV} \leq 1}  \frac{1}{n} \sum_{i=1}^n \epsilon_i \int_{\mathbb{R}^d} \int_{\mathbb{S}^{d}} \sigma(\langle w, \tilde{x}_i \rangle)^2 \nu(dw) \mu_i(dx_i) \right] \\
    & \leq \delta \E \left[\sup_{\|w\|_2 \leq 1}  \frac{1}{n} \sum_{i=1}^n \epsilon_i \E_{\mu_i}[\sigma(\langle w, \tilde{x}_i \rangle)^2] \right] \\
    & \leq \delta \E_{\mu_1, \dots, \mu_n} \left[ \E \left[\sup_{\|w\|_2 \leq 1}  \frac{1}{n} \sum_{i=1}^n \epsilon_i \sigma(\langle w, \tilde{x}_i \rangle)^2 \middle| x_1, \dots, x_n \right] \right]~, \\
\end{align*}
where the last step uses Jensen's inequality and Fubini's theorem.  The conditional expectation is itself a Rademacher complexity, so we may apply the contraction lemma again as the $\sigma(\langle w, \tilde{x}_i \rangle)^2$ activation is $2\sqrt{2}R$-Lipschitz for the domain $\I$ of $\tilde{x}_i$.  Using the variational definition of the $l_2$ norm we have the bound:
\begin{align*}
    \E \left[\sup_{\|f\|_{\mathcal{S}_1} \leq \delta} \frac{1}{n} \sum_{i=1}^n \epsilon_i f(\mu_i) \right] \leq \frac{ 4R^2\delta}{\sqrt{n}}~.
\end{align*}
The high probability bound then follows from McDiarmid's inequality.

\end{proof}

\subsection{Proof of Proposition~\ref{prop:concentration}}

\begin{proof}

We appeal to the following concentration inequality for empirical measures under the Wasserstein metric:

\begin{theorem}[Theorem 1 in~\textcite{fournier2015rate}]
    Let $\hat{\mu}_N = \frac{1}{N} \sum_{j=1}^N \delta_{X_j}$ where $X_i \sim \mu \in \mathcal{P}(\I)$ iid.  Then $\E[W_1(\hat{\mu}_N, \mu)] \lesssim N^{-1/d}$ where $d > 2$ is the dimension of $\I$.
\end{theorem}

It's easy to see that any $\phi \in \A_2$ has Lipschitz constant bounded above by $2\sqrt{2}R$, and therefore $\sup_{\phi \in \A_2}|\langle \phi, \mu - \mu^* \rangle| \leq 2\sqrt{2}R W_1(\mu, \mu^*)$.  Therefore

\begin{align*}
    \E \left[\sup_{\phi \in \mathcal{A}} \frac{1}{n} \sum_{i=1}^n \epsilon_i \left  \langle \phi,    \mu_i \right\rangle \right]
    & \leq \E \left[\sup_{\phi \in \mathcal{A}} \frac{1}{n} \sum_{i=1}^n \epsilon_i\left\langle \phi,    \mu^* \right\rangle \right] + \E \left[\sup_{\phi \in \mathcal{A}} \frac{1}{n} \sum_{i=1}^n \epsilon_i \left\langle \phi,    (\mu^* - \mu_i) \right\rangle \right] \\
    & \leq 2R^2 \E \left[\left|\frac{1}{n} \sum_{i=1}^n \epsilon_i \right| \right] + 2\sqrt{2}R \E[W_1(\mu_i, \mu^*)] \\
    & \lesssim R^2(n^{-1/2} + \E_{N \sim \Omega}[N^{-1/d}])~.
\end{align*}
The conclusion then follows from the same Rademacher decomposition as in Proposition~\ref{prop:rad}.

\end{proof}

\subsection{Proof of Theorem \ref{thm:inclusion}}

For simplicity, we consider spherical inputs rather than Euclidean inputs, and so we may consider $k(x,y) = \int_{\mathbb{S}^d} \sigma(\langle w, x \rangle) \sigma(\langle w, y \rangle) \kappa(d w)$ without the $\tilde{x}$ bias terms, and assume $x \in \mathbb{S}^d$.  Note that the Euclidean inputs may be seen as a restriction of the spherical inputs to an appropriate spherical cap, see~\textcite{bach2017breaking} for details of this construction.


\subsubsection{Spherical Harmonics and Kernel Norm Background}

We'll use $\simeq$ to denote equality up to universal constants.  To understand functions in $\A_2$, we require the following details of spherical harmonics~\cite{efthimiou2014spherical}.

A basis on $\mathbb{S}^d$ is given by the orthogonal polynomials $Y_{k,j}$, where $k \geq 0$ and $ 1 \leq j \leq N(d,k)$ where

\begin{align*}
    N(d,k) &\simeq \frac{k+d}{k} \frac{\Gamma(k+d-1)}{\Gamma(d)\Gamma(k)}\\
        &\simeq \frac{k+d}{k} \frac{(k+d)^{k+d - 3/2}}{d^{d-1/2}k^{k-1/2}}
\end{align*}

The Legendre polynomials $P_k(t)$ act on one dimensional real inputs and satisfy the addition formula

\begin{align*}
    \sum_{j=1}^{N(d,k)} Y_{k,j}(x) Y_{k,j}(y) = N(d,k) P_k(\langle x, y \rangle)
\end{align*}

Finally, given a function $g: \mathbb{S}^d \rightarrow \mathbb{R}$, the $k$th spherical harmonic of $g$ is the degree $k$ component of $g$ in the orthogonal basis, equivalently written as

\begin{align*}
    g_k(x) = N(d,k) \int_{\mathbb{S}^d} g(y) P_k(\langle x, y \rangle) \kappa(dy)
\end{align*}

We also require several calculations on functions with bounded functional norm and projections~\cite{bach2017breaking}, where we remind that we're using the activation $\sigma(x)^2$.  For $g \in \A_2$ or $g(x) = \sigma(\langle w, x \rangle)^2$ for any $w \in \mathbb{S}^d$, we have that $g_{2k} = 0$ for all $k \geq 2$.

For $g \in \A_2$, the norm of each harmonic satisfies $\|g_k\|_2^2 = \lambda_k^2 N(d,k)$, and the kernel norm can be calculated explicitly as

\begin{align*}
    \gamma_2(g)^2 = \sum_{k=0, \lambda_k \neq 0}^\infty \lambda_k^{-2} \|g_k\|_{L_2}^2
\end{align*}

We have that $\lambda_1 \simeq d^{-1}$, $\lambda_k = 0$ for $k \geq 3$ and $k$ even, and for $k \geq 3$ and $k$ odd:

\begin{equation}
    \lambda_k \simeq \pm \frac{d^{d/2 + 1/2} k^{k/2 - 3/2}}{(d+k)^{k/2 + d/2 + 1}}
\end{equation}

\subsubsection{Separation of $\S_1$ and $\S_2$}

Let $g(x) = \sigma(\langle x, w \rangle)^2$ for an arbitrary $w \in \mathbb{S}^d$, we have that $\|g_k\|_2^2 = \lambda_k^2 N(d,k)$.  Define $\tilde{g} = g - \sum_{i=0}^{d^2-1} g_i$.

The following lemmas capture that $\tilde{g}$ has high correlation with $g$ and exponentially small correlation with functions in $\A_2$.

\begin{lemma}\label{lem:correlation-lower-bound}  
    The correlation lower bound $\langle g, \tilde{g} \rangle \gtrsim d^{-21/2}$ holds.
\end{lemma}

\begin{proof}

Note that

\begin{equation}
    \langle g, \tilde{g} \rangle = \sum_{k=d^2} \|g_k\|_2^2 = \sum_{k=d^2} \lambda_k^2 N(d,k)
\end{equation}

We can calculate, because $k+d \leq 2k$:

\begin{align*}
    \lambda_k^2 N(d,k) & \simeq \frac{d^{d + 1} k^{k - 3}}{(d+k)^{k + d + 2}} \cdot \frac{k+d}{k} \frac{(k+d)^{k+d - 3/2}}{d^{d-1/2}k^{k-1/2}} \\
    & \simeq d^{3/2} k^{-7/2} (k+d)^{-7/2}\\
    & \gtrsim d^{3/2} k^{-7} 
\end{align*}

And therefore
\begin{align*}
    \langle g, \tilde{g} \rangle \gtrsim \sum_{k=d^2}^\infty d^{3/2} k^{-7}  \geq d^{3/2} \int_{d^2}^\infty k^{-7} dk \simeq d^{3/2} (d^2)^{-6}
\end{align*}

which yields the desired lower bound.

\end{proof}

\begin{lemma}\label{lem:correlation-upper-bound}  
    The value of the optimization problem
    \begin{equation*}
    \begin{aligned}
    \max_{\phi} \quad & \langle \phi, \tilde{g} \rangle_{L_2}\\
    \textrm{s.t.} \quad & \gamma_2(\phi)^2 \leq \delta^2
    \end{aligned}
    \end{equation*}
    is upper bounded by $\delta \cdot d^{1/2-d/3}$
\end{lemma}

\begin{proof}
    By orthogonality we may assume $\phi_k = \alpha_k \tilde{g}_k = \alpha_k g_k$, where $\alpha_k = 0$ for $k < d^2$.  Then the problem is equivalently
    
    \begin{equation*}
    \begin{aligned}
    \min_{\alpha} \quad & - \sum_{k=d^2}^\infty \alpha_k \|g_k\|_2^2 \\
    \textrm{s.t.} \quad & \sum_{k=d^2}^\infty \alpha_k^2 \lambda_k^{-2} \|g_k\|_2^2 \leq \delta^2
    \end{aligned}
    \end{equation*}
    
    Taking $\lambda$ as a Lagrangian multiplier yields the optimality condition $\alpha_k = (2\lambda)^{-1} \lambda_k^2$.
    
    Plugging this into the constraint and introducing notation $S$ yields 
    
    \begin{align*}
        (2\lambda)^{-2} S : = (2\lambda)^{-2} \sum_{k=d^2}^\infty \lambda_k^2 \|g_k\|_2^2 \leq \delta^2
    \end{align*}
    
    Then the objective (returned to a maximum) obeys the bound
    
    \begin{align*}
        \sum_{k=d^2} (2\lambda)^{-1} \lambda_k^2 \|g_k\|_2^2 & = (2\lambda)^{-1} S\\
        & \leq \delta \sqrt{S}
    \end{align*}
    
    So it remains to calculate $S$.  Plugging in the value of $\|g_k\|_2^2$ gives
    
    \begin{equation*}
        S = \sum_{k=d^2}^\infty \lambda_k^4 N(d,k)
    \end{equation*}
    
    We can give the form of each term, using that $k \geq d^2$:
    
    \begin{align*}
        \lambda_k^4 N(d,k) & \lesssim d^{3/2} k^{-7} \frac{d^{d + 1} k^{k - 3}}{(d+k)^{k + d + 2}}\\
        & \lesssim d^{3/2} k^{-7} \frac{d^{d + 1} k^{k - 3}}{k^{k + d + 2}}\\
        & \lesssim d^{5/2} k^{-12} \left(\frac{d}{k} \right)^d\\
        & \lesssim d^{5/2} k^{-12} \left(\frac{d}{k^{1/2}} \cdot \frac{1}{k^{1/2}}\right)^d\\
        & \lesssim d^{5/2} k^{-12} k^{-d/2}\\
    \end{align*}
    
    For sufficiently large $d$, we may ignore the lower terms and reduce the exponential term to $k^{-d/3}$, then:
    
    \begin{align*}
        S \lesssim \sum_{k=d^2}^\infty k^{-d/3} \simeq \int_{d^2}^\infty k^{-d/3} \simeq d^{-1}(d^2)^{1-d/3}
    \end{align*}
    
    The bound follows.
    
\end{proof}

Let $h = g - g_0 - g_2$, and define $f_1(\mu) = d^{-1}\sigma(\langle h, \mu \rangle)$, remembering that we're using the regular ReLU for the measure network activation.

\begin{lemma}
    $\|f_1\|_{\S_1} \lesssim 1$.
\end{lemma}
\begin{proof}
    It suffices to bound $\gamma_1(h)$, remembering that our test functions are defined using networks with the squared ReLU activation.  Clearly $\gamma_1(g) \leq 1$ as it itself a single neuron.  For the other terms, we can write the harmonics explicitly, using the fact that $P_0(t) = 1$ and $P_2(t) = \frac{(d+1)t^2 - 1}{d}$.  Starting with the constant term $g_0$:
    
    \begin{align*}
        g_0(x) &= \int_{\mathbb{S}^d} g(y) \kappa(dy)\\
        & = \int_{\mathbb{S}^d} \sigma(\langle w, y \rangle)^2 \kappa(dy)\\
        & = \int_{\mathbb{S}^d} \sigma(y_1)^2 \kappa(dy)\\
        & = \frac{1}{2(d+1)}
    \end{align*}
    
    Note that $\sigma(z)^2 + \sigma(-z)^2 = z^2$, so we can represent a constant function as a neural network via:
    
    \begin{align*}
        \sum_{i=1}^{d+1} \sigma(\langle e_i, x\rangle)^2 + \sigma(\langle -e_i, x\rangle)^2 & = \sum_{i=1}^{d+1} \langle e_i, x \rangle^2\\
        & = \|x\|_2 = 1
    \end{align*}
    
    So we have $\gamma_1(g_0) \leq 1$.
    
    The second spherical harmonic is given as:
    
    \begin{align*}
        g_2(x) &= N(d,2) \int_{\mathbb{S}^d} g(y) \frac{(d+1)\langle x, y \rangle^2 - 1}{d} \kappa(dy)\\
        &= \frac{N(d,2)}{d} \left((d+1)\int_{\mathbb{S}^d} g(y) \langle x, y \rangle^2 \kappa(dy) - \int_{\mathbb{S}^d} g(y) \kappa(dy) \right)
    \end{align*}
    
    We can represent the constant term as above, and the first integral as
    
    \begin{align*}
        \int_{\mathbb{S}^d} \sigma(\langle w, y \rangle)^2 \langle x, y \rangle^2 \kappa(dy) & = \int_{\mathbb{S}^d} \sigma(\langle w, y \rangle)^2 (\sigma(\langle x, y \rangle)^2 + \sigma(\langle x, -y \rangle)^2) \kappa(dy) \\
        & = \int_{\mathbb{S}^d} \sigma(\langle x, y \rangle)^2 (\sigma(\langle w, y \rangle)^2 + \sigma(\langle w, -y \rangle)^2 ) \kappa(dy) \\
        & = \int_{\mathbb{S}^d} \sigma(\langle x, y \rangle)^2 \langle w, y \rangle^2 \kappa(dy)
    \end{align*}
    
    This last line is a convex neural network representation using the squared ReLU activation, and thus we have $\gamma_1\left(\int_{\mathbb{S}^d} g(y) \langle x, y \rangle^2 \kappa(dy)\right) \leq \int_{\mathbb{S}^d} \langle w, y \rangle^2 \kappa(dy) = \frac{1}{d+1}$.
    
    Thus, $\gamma_1(g_2) \leq \frac{N(d,2)}{d} \left(1 + 1\right) \lesssim d$.  And all together, $\gamma_1(h) \leq \gamma_1(g) + \gamma_1(g_0) + \gamma(g_2) \lesssim d$.
    
    So by homogeniety the bound on $\|f\|_{\S_1}$ follows.
    
\end{proof}

Our choice of $f_1$ induces a separation between $S_1$ and $S_2$.

\begin{theorem}\label{thm:s1s2}
    We have that $\|f_1\|_{\S_1} \lesssim 1$, and 
    \begin{equation}
        \inf_{\|f\|_{\S_2} \leq \delta} \|f - f_1\|_\infty \gtrsim |d^{-11} - d^{-d/3}\delta|
    \end{equation}
\end{theorem}

\begin{proof}

    Because we've subtracted out the $0$th and $2$nd harmonics, and all other even harmonics are zero, $\tilde{g}$ and $h$ are odd functions.
    
    Consider the signed measure $\nu(dx) := \frac{2\tilde{g}(x)}{\|\tilde{g}\|_{L_1}} \kappa(dx)$, with Jordan decomposition $\nu = \nu^+ - \nu^-$ with the positive measures $\nu^+(dx) := \frac{2\sigma(\tilde{g}(x))}{\|\tilde{g}\|_{L_1}} \kappa(dx)$ and $\nu^-(dx) := \frac{2\sigma(-\tilde{g}(x))}{\|\tilde{g}\|_{L_1}} \kappa(dx)$.
    
    Note that from the oddness of $\tilde{g}$ and symmetry of $\kappa$:
    
    \begin{align*}
        TV(\nu^-) & = \frac{2}{\|\tilde{g}\|_{L_1}}\int_{\mathbb{S}^d} \sigma(-\tilde{g}(x)) \kappa(dx)\\
        & = \frac{2}{\|\tilde{g}\|_{L_1}}\int_{\mathbb{S}^d} \sigma(\tilde{g}(-x)) \kappa(dx)\\
        & = \frac{2}{\|\tilde{g}\|_{L_1}}\int_{\mathbb{S}^d} \sigma(\tilde{g}(x)) \kappa(dx)\\
        & = TV(\nu^+)
    \end{align*}
    
    Because $TV(\nu^+) + TV(\nu^-) = TV(\nu) = 2$, we conclude $\nu^+$ and $\nu^-$ are both probability measures.  We'll use these measures to separate $f$ and $f_1$.  By Lipschitz continuity of $\sigma$:
    
    \begin{align*}
        |f(\nu^+) - f(\nu^-)| & = \left|\int_{\mathbb{S}^d} \sigma(\langle \phi, \nu^+ \rangle) - \sigma(\langle \phi, \nu^- \rangle) \chi(d \phi) \right|\\
        & \leq \int_{\mathbb{S}^d} |\sigma(\langle \phi, \nu + \nu^- \rangle) - \sigma(\langle \phi, \nu^- \rangle)| \chi(d \phi)\\
        & \leq \sup_{\gamma_2(\phi) \leq 1} |\langle \phi, \nu \rangle| \|f\|_{\S_2} \\
        & \leq \frac{2}{\|\tilde{g}\|_{L_1}} \sup_{\gamma_2(\phi) \leq 1} |\langle \phi, \tilde{g} \rangle| \|f\|_{\S_2} \\
        & \lesssim \frac{2}{\|\tilde{g}\|_{L_1}} d^{1/2-d/3} \delta
    \end{align*}
    
    where in the last line we use Lemma~\ref{lem:correlation-upper-bound}.
    
    Concerning the function $f_1$, we first use oddness again to notice:
    
    \begin{align*}
        \langle h, \nu^- \rangle & = \frac{2}{\|\tilde{g}\|_{L_1}} \int_{\mathbb{S}^d} h(x) \sigma(-\tilde{g}(x)) \kappa(dx)\\
        & = \frac{2}{\|\tilde{g}\|_{L_1}} \int_{\mathbb{S}^d} h(x) \sigma(\tilde{g}(-x)) \kappa(dx)\\
        & = \frac{2}{\|\tilde{g}\|_{L_1}} \int_{\mathbb{S}^d} h(-x) \sigma(\tilde{g}(x)) \kappa(dx)\\
        & = -\langle h, \nu^+ \rangle
    \end{align*}
    
    So $\langle h, \nu \rangle = \langle h, \nu^+ - \nu^-\rangle = 2 \langle h, \nu^+ \rangle$, and therefore from Lemma~\ref{lem:correlation-lower-bound} with $\alpha = 2$,
    
    \begin{align*}
    d^{-21/2} \lesssim \langle g, \tilde{g} \rangle & = \langle h, \tilde{g} \rangle\\
    & = \frac{\|\tilde{g}\|_{L_1}}{2} \langle h, \nu \rangle\\
    & = \|\tilde{g}\|_{L_1} \langle h, \nu^+ \rangle
    \end{align*}
    
    So $\langle h, \nu^+ \rangle \gtrsim \frac{d^{-21/2}}{\|\tilde{g}\|_{L_1}} $, and we conclude
    
    \begin{align*}
        |f_1(\nu^+) - f_1(\nu^-)| & = d^{-1}|\sigma(\langle h, \nu^+ \rangle) - \sigma(\langle h, \nu^- \rangle)|\\
        &= d^{-1}\sigma(\langle h, \nu^+\rangle)\\
        & \gtrsim \frac{d^{-23/2}}{\|\tilde{g}\|_{L_1}}
    \end{align*}
    
    Now, suppose $\|f-f_1\|_\infty \leq \epsilon$.  Then 
    \begin{align*}
        \frac{d^{-23/2}}{\|\tilde{g}\|_{L_1}} & \lesssim |f_1(\nu^+) - f_1(\nu^-)|\\
        & \leq |f_1(\nu^+) - f(\nu^+)| + |f(\nu^+) - f(\nu^-)| + |f(\nu^-) - f_1(\nu^-)|\\
        & \lesssim \epsilon + \frac{2}{\|\tilde{g}\|_{L_1}} d^{1/2-d/3} \delta + \epsilon
    \end{align*}
    
    So for sufficiently large $d$, we have $\frac{|d^{-23/2} - d^{1/2-d/3}\delta|}{\|\tilde{g}\|_{L_1}} \lesssim \epsilon$.  Finally, note by Jensen's inequality and spherical harmonic orthogonality that $\|\tilde{g}\|_{L_1} \leq \|\tilde{g}\|_{L_2} \leq \|g\|_{L_2} \lesssim d^{-1/2}$.
    
\end{proof}

\subsubsection{Separation of $\S_2$ and $\S_3$}

In order to instantiate the class $\S_3$, we must fix $\tau$, the base probability measure over test functions in $\A_2$. Consider some probability distribution $\zeta$ over the square-summable sequences $l_2(\mathbb{R^+})$ such that for $c \in supp(\zeta)$, $\sum_{k=0}^\infty c_k^2 = 1$.  Furthermore, we will make the simplyfing assumption that $c_0 = 0$.  For each $k$ let $\kappa_k$ be uniform over $\mathbb{S}^{N(d,k)-1}$, and note that $N(d,1) = d+1$ so $\kappa = \kappa_1$.

Then we sample $\phi \sim \tau$ as $\phi = \sum_{k=1}^\infty \sum_{j=0}^{N(d,k)} \lambda_k c_k \alpha_{kj} Y_{kj}$ where $c \sim \zeta$ and $\alpha_k \sim \kappa_k$.  Observe that $$\gamma_2(\phi)^2 = \sum_{k=1, \lambda_k \neq 0}^\infty \sum_{j=1}^{N(d,k)} \lambda_k^{-2} \lambda_k^2 c_k^2 \alpha_{kj}^2 = 1$$ so $\tau$ indeed samples functions from $\A_2$.

We define $f_2(\mu) = \sigma(\langle g, \mu \rangle)$ where $g = \lambda_{1} Y_{1,1}$.  Clearly $\gamma_2(g)^2 = \lambda_{1}^{-2} \lambda_{1}^2 \|Y_{1,1}\|_{L_2}^2 = 1$, so $\|f_2\|_{\S_2} \leq 1$.

\begin{theorem}\label{thm:s2s3}
    We have that $\|f_2\|_{\S_2} \leq 1$, and 
    \begin{equation}
        \inf_{\|f\|_{\S_3} \leq \delta} \|f - f_2\|_\infty \gtrsim d^{-2} \delta^{-5/d}
    \end{equation}
\end{theorem}

\begin{proof}
    Consider the function $h(x) = \sum_{j=1}^{N(d,1)} \beta_{1,j} Y_{1,j}$ and probability measure $\mu_\beta^*(dx) = \frac{h(x) + \|h\|_\infty}{\|h + \|h\|_\infty\|_{L_1}} \kappa(dx)$. Observe that
    
    \begin{align*}
        f_2(\mu_\beta^*) = \frac{\lambda_{1}}{\|h + \|h\|_\infty\|_{L_1}} \sigma(\langle e_1, \beta \rangle)
    \end{align*}
    
    For a function $f \in \S_3$ with density $q$ with respect to $\tau$, we have:
    
    \begin{align*}
        f(\mu_\beta^*) &= \int_{\A_2} \sigma(\langle \phi, \mu_\beta^* \rangle) q(\phi) \tau(d\phi) \\
        & = \frac{\lambda_{1}}{\|h + \|h\|_\infty\|_{L_1}} \int_{l_2(\mathbb{R^+})} \int_{\mathbb{S}^d} \sigma(\langle c_1 \alpha_1, \beta \rangle) \hat{q}(c, \alpha_1) \kappa(d \alpha_1) \zeta(d c)\\
        & = \frac{\lambda_{1}}{\|h + \|h\|_\infty\|_{L_1}} \int_{\mathbb{S}^d} \sigma(\langle \alpha_1, \beta \rangle) \left[ \int_{l_2(\mathbb{R^+})} c_1 \hat{q}(c, \alpha_1) \zeta(d c) \right] \kappa(d \alpha_1) \\
    \end{align*}
    
    where $\hat{q}$ marginalizes out all other $\alpha_k$ terms.  Let $\tilde{q}(\alpha_1) = \int_{l_2(\mathbb{R^+})} c_1 \hat{q}(c, \alpha_1) \zeta(d c)$.  From the fact that $c_1 \leq 1$, and by Jensen's inequality, $\|\tilde{q}\|_{L_2(\kappa)} \leq \|\hat{q}\|_{L_2(\kappa \times \zeta)} \leq \|q\|_{L_2(\tau)}$.
    
    Now we may appeal to a separation of test function representations acting on spherical inputs.  From D.5 in~\cite{bach2017breaking}, there exists some $\beta \in \mathbb{S}^d$ such that \begin{align*}
        |\sigma(\langle e_1, \beta \rangle) - \int_{\mathbb{S}^d} \sigma(\alpha_1, \beta) \tilde{q}(\alpha_1) \kappa(d\alpha_1)\| & \gtrsim \|\tilde{q}\|_{L_2}^{-5/d} \geq \|q\|_{L_2}^{-5/d}
    \end{align*}
    
    Therefore
    
    \begin{align*}
        |f_2(\mu_\beta^*) - f(\mu_\beta^*)| \gtrsim  \frac{\lambda_{1}}{\|h + \|h\|_\infty\|_{L_1}}\|q\|_{L_2}^{-5/d}
    \end{align*}
    
    Finally, note that $\lambda_1 \simeq d^{-1}$, and by the addition formula and the fact $P_k(1) = 1$ for all $k$:
    
    \begin{align*}
        \|h + \|h\|_\infty\|_{L_1} & \leq 2 \|h\|_\infty\\
        & = 2 \max_{x \in \mathbb{S}^d} \sum_{j=1}^{N(d,1)} \beta_{1,j} Y_{1,j}(x)\\
        & \leq 2 \max_{x \in \mathbb{S}^d} \|\beta\|_2 \sqrt{\sum_{j=1}^{N(d,1)} Y_{1,j}(x)^2}\\
        & \leq 2 N(d,1)\\
        & \lesssim d
    \end{align*}
    
    So we arrive at the desired bound.

\end{proof}
\section{Experimental Details and Additional Data}

\paragraph{Synthetic Details:}

For all experiments we use the same architecture.  Namely, for an input set $x = (x_1, \dots, x_N)$, the network is defined as $f_N(x) = w_3^T\sigma(W_2 \frac{1}{N}\sum_{i=1}^N \sigma(W_1 \tilde{x}_i))$, where we choose the architecture as $W_1 \in \mathbb{R}^{h_1 \times d}$, $W_2 \in \mathbb{R}^{h_2 \times h_1}$, and $w_3 \in \mathbb{R}^{h_2}$.  Here, $h_1, h_2 = 100$ for $\S_1$, $h_1 = 100$ and $h_2 = 1000$ for $\S_2$, and $h_1 = h_2 = 1000$ for $\S_3$.  The weights are initialized with the uniform Kaiming initialization~\parencite{he2015delving} and frozen as described in Table~\ref{tab:freeze}.

We relax the functional norm constraints to penalties, by introducing regularizers of the form $\lambda \|f_N\|_{\S_i}$ for $\lambda$ a hyperparameter.  Let $K(\cdot)$ map a matrix to the vector of row-wise squared norms, and let $|\cdot|$ denote the element-wise absolute value of a matrix.  Then we calculate the functional norms via the path norm as follows:
\begin{itemize}
    \item For $\S_1$, $\|f_N\|_{\S_1} = |w_3|^T|W_2|K(W_1)$
    \item For $\S_2$, we explicitly normalize the frozen matrix $W_1$ to have all row-wise norms equal to 1, then $\|f_N\|_{\S_2} = |w_3|^TK(W_2)$
    \item For $\S_3$, we normalize the rows of $W_1$ and $W_2$, which simply implies $\|f_N\|_{\S_3} = \|w_3\|_2$
\end{itemize}

We optimized via Adam~\parencite{kingma2014adam} with an initial learning rate of 0.0005, for 5000 iterations.  Under this architecture, all $\S_1$, $\S_2$ and $\S_3$ functions achieved less than $10^{-15}$ training error without regularization on all objective functions (listed below) on training sets of 100 samples.

We use the following symmetric functions for our experiments:
\begin{itemize}
    \item $f_N^*(x) = \mathrm{max}_i(\|x_i\|_2^{-1})$
    \item $f_N^*(x) = \lambda \log \left(\sum_{i=1}^N \exp(\|x_i\|_2^{-1} / \lambda) \right)$ for $\lambda = 0.1$
    \item $f_N^*(x) = \mathrm{median}(\{\|x_i\|_2^{-1}\}_{i=1}^N)$
    \item $f_N^*(x) = \mathrm{second}_i(\|x_i\|_2^{-1})$ i.e. the second largest value in a given set
    \item $f_N^*(x) = \frac{1}{N} \sum_{i=1}^N (\|x_i\|_2^{-1})$
    \item $f_N^*(x) = \frac{2}{N(N-1)} \sum_{i < j} \frac{1}{\|x_i - x_j\|_2}$
    \item $f_N^*(x)$ is an individual neuron, parameterized the same as $f_N$ but with different hidden layer sizes.  For the neuron, $h_1 = h_2 = 1$, for the smooth\_neuron, $h_1 = 100$ and $h_2 = 1$.  Additionally, the proof of Theorem~\ref{thm:s1s2} dictates that we must choose the neuron's test function to have large kernel norm, so we initialize $W_1$ elementwise from the Gaussian mixture with density $0.5 * \mathcal{N}(1, 0.5) + 0.5 * \mathcal{N}(-1, 0.5)$.
\end{itemize}

Note that in order to guarantee the ``smooth\_neuron" is representable by our finite-width networks, we explicitly set $W_1$ in the $\S_2$ and $\S_3$ models to equal the $W_1$ matrix of the ``smooth\_neuron".

For each model in each experiment, $\lambda$ was determined through cross validation over $\lambda \in [0, 10^{-6}, 10^{-4}, 10^{-2}]$ using fresh samples of training data, and choosing the value of $\lambda$ with lowest generalization error, which was calculated from another 1000 sampled points.

Then, with determined $\lambda$, each model was trained from scratch over 10 runs with independent random initializations.  The mean and standard deviation of the generalization error, testing on varying values of $N$, are plotted in Figure~\ref{fig:plots}.

\paragraph{Application Details:}

For the MNIST experiment with results given in Table~\ref{tab:mnist}, we follow a similar setup to~\textcite{de2019stochastic}.  From an image in $\mathbb{R}^{28 \times 28}$, we produce a point cloud by considering a set of tuples of the form $(r, c, t)$, which are the row, column and intensity respectively for each pixel.  We restrict to pixels where $t > 0.5$, and select the pixels with the top 200 intensities to comprise the point cloud (if there are fewer than 200 pixels remaining after thresholding, we resample among them).  Furthermore, we normalize the row and column values among all the points in the cloud.  This process maps an image to a set $S \subseteq \mathbb{R}^3$ such that $|S| = 200$.  

For this dataset we consider $h_1 = 500$ and $h_2 = 500$ for our $\S_i$ finite-width architectures.


We perform cross-validation by setting aside $10\%$ of the data as a validation set, and calculate the mean and standard deviation of the generalization error over five runs.  In order to study generalization in this setting, we test on point clouds of different size, $100$ and $200$, and show the results in Table~\ref{tab:mnist}.  The starting learning rate is $0.001$.  Otherwise, all other experimental details are the same as above.

\paragraph{Robust Mean Details:}

We use the regular ReLU activation in the first layer for training stability.  Each network is trained on a batch of 5000 input sets sampled as above, as the task of robust estimation appears more susceptible to overfitting than the simpler symmetric objectives learned in the previous section.  All networks are trained for 30000 iterations, and all other details of training are kept consistent with the previous section (including the larger number of random kernel features).

The hyperparameters required for the adversarial estimator in~\cite{diakonikolas2017being} are $\tau$ and ``cher", which both control the thresholding of which vectors are discarded according to the projection on the maximal eigenvector of the empirical covariance.  Cross validation over the sets $[0.1, 0.15, 0.2]$ and $[1.5, 1.8, 2.0, 2.3]$ yielded the choices $\tau = 0.1$ and ``cher" $= 1.5$.

\paragraph{Additional Experiments}

In Figure~\ref{fig:plots_extra} we consider higher dimensional vectors for our set inputs to the symmetric models.  In Figure~\ref{fig:plots_multi} we consider training over multiple set sizes as well, with the input size sampled uniformly from ${4, 5, 6}$.

\begin{figure}
\centering

\begin{subfigure}{.4\textwidth}
  \centering
  \includegraphics[width=1.\linewidth]{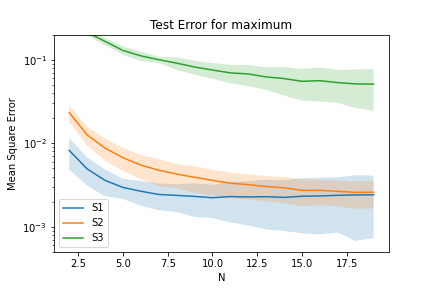}
\end{subfigure}%
\begin{subfigure}{.4\textwidth}
  \centering
  \includegraphics[width=1.\linewidth]{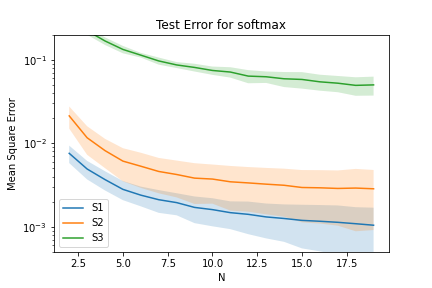}
\end{subfigure}%

\begin{subfigure}{.4\textwidth}
  \centering
  \includegraphics[width=1.\linewidth]{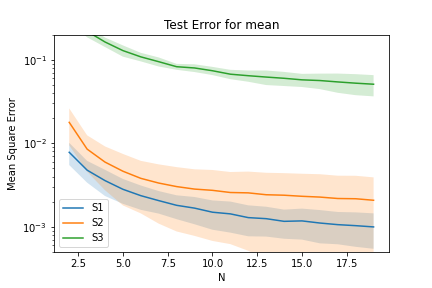}
\end{subfigure}%
\begin{subfigure}{.4\textwidth}
  \centering
  \includegraphics[width=1.\linewidth]{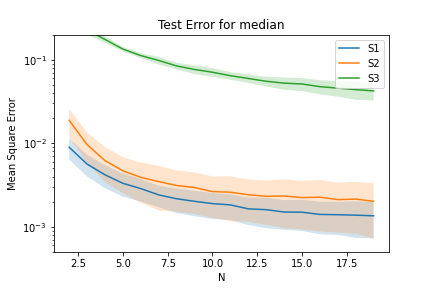}
\end{subfigure}%

\begin{subfigure}{.4\textwidth}
  \centering
  \includegraphics[width=1.\linewidth]{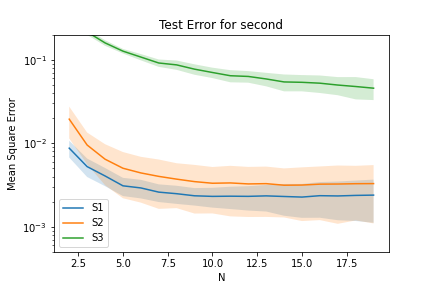}
\end{subfigure}%
\begin{subfigure}{.4\textwidth}
  \centering
  \includegraphics[width=1.\linewidth]{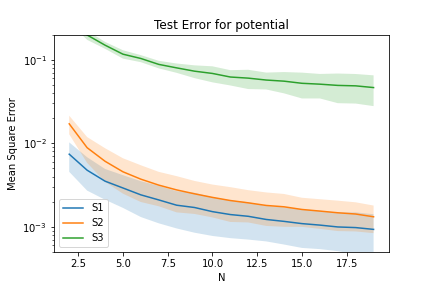}
\end{subfigure}%

\begin{subfigure}{.4\textwidth}
  \centering
  \includegraphics[width=1.\linewidth]{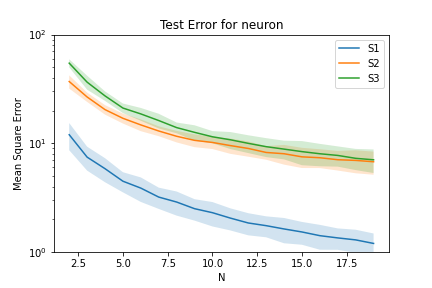}
\end{subfigure}%
\begin{subfigure}{.4\textwidth}
  \centering
  \includegraphics[width=1.\linewidth]{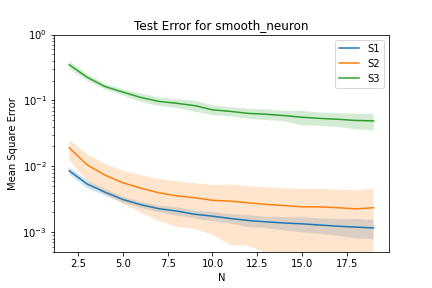}
\end{subfigure}%
\caption{Test Error for $d = 20$}
\label{fig:plots_extra}
\end{figure}

\begin{figure}
\centering

\begin{subfigure}{.4\textwidth}
  \centering
  \includegraphics[width=1.\linewidth]{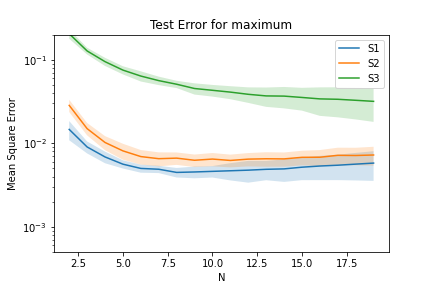}
\end{subfigure}%
\begin{subfigure}{.4\textwidth}
  \centering
  \includegraphics[width=1.\linewidth]{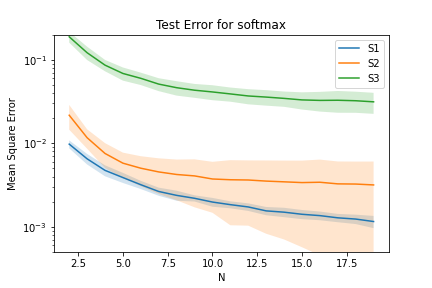}
\end{subfigure}%

\begin{subfigure}{.4\textwidth}
  \centering
  \includegraphics[width=1.\linewidth]{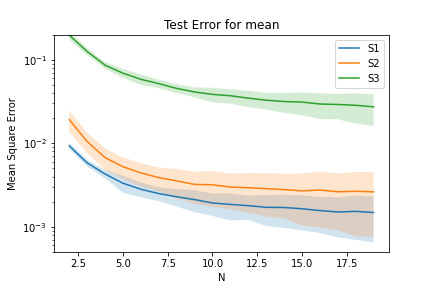}
\end{subfigure}%
\begin{subfigure}{.4\textwidth}
  \centering
  \includegraphics[width=1.\linewidth]{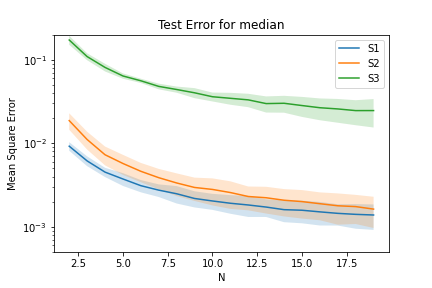}
\end{subfigure}%

\begin{subfigure}{.4\textwidth}
  \centering
  \includegraphics[width=1.\linewidth]{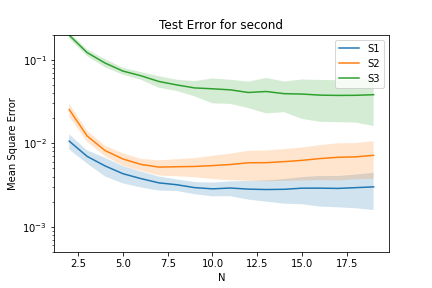}
\end{subfigure}%
\begin{subfigure}{.4\textwidth}
  \centering
  \includegraphics[width=1.\linewidth]{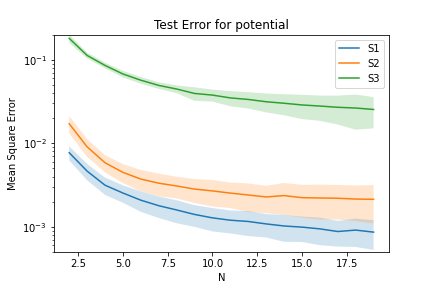}
\end{subfigure}%

\begin{subfigure}{.4\textwidth}
  \centering
  \includegraphics[width=1.\linewidth]{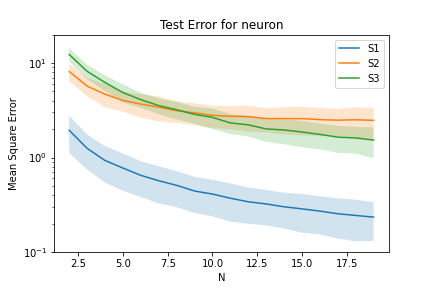}
\end{subfigure}%
\begin{subfigure}{.4\textwidth}
  \centering
  \includegraphics[width=1.\linewidth]{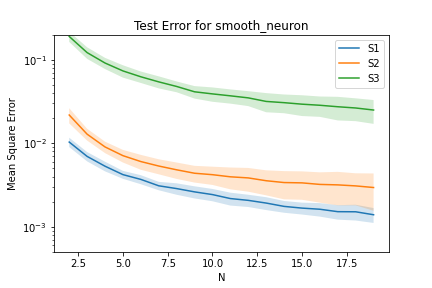}
\end{subfigure}%
\caption{Test Error for varied input size training}
\label{fig:plots_multi}
\end{figure}

We consider the Pointcloud MNIST dataset, after mapping our image to sets.  This dataset is substantially more difficult than regular MNIST, as the induced sets obfuscate the geometric structure of the original images.  The results on Pointcloud MNIST, across differently-sized set representations of images, are given in Table~\ref{tab:mnist}.  The fact that we only consider three-layer networks limits the ability of the model to reconstruct the original image representation and perform comparably to a model acting on regular MNIST.  Nevertheless, we still observe the expected ordering of our functional spaces.  When testing on smaller sets than training, the generalization error increases faster for $\S_3$ than for $\S_1$ and $\S_2$.

\begin{table}
    \centering
    \begin{tabular}{p{5pt}c|c|c}
         & &  Error ($N = 100$) &  Error ($N = 200$)   \\
         \hline
    & $\mathcal{S}_1$  &  $8.03$ & $5.62$ \\
     & $\mathcal{S}_2$  &  $8.25$ & $5.78$ \\
      & $\mathcal{S}_3$  &   $14.45$ & $10.80$ \\
    \end{tabular}
    \caption{Classification test error on Pointcloud MNIST in percent, after images are compressed into sets of size $N$, trained with $N = 200$.}  
    \label{tab:mnist}
\end{table}

In Table~\ref{tab:robust_big_n} we consider the robust mean experiment, using the same hyperparameters except training on sets of larger size ($N = 60$) and plotting MSE on sets of varying size.  As with the smaller scale experiment, we observe that $\S_1$ enjoys a slight advantage over the other methods when restricting attention to the in-distribution generalization setting of $N = 60$, but outside that range the performance is comparable to the naive sample mean, suggesting that out-of-distribution generalization for the robust mean is not easily attainable for these networks.

\begin{table*}[h]
    \centering
    \begin{tabular}{p{5pt}c|c|c|c|c|c}
         & &  $N = 20$ &  $N = 40$ & $N = 60$ & $N = 80$ & $N = 100$   \\
         \hline
    & $\mathcal{S}_1$ & $0.149 \pm 0.039$ & $0.073 \pm 0.023$ & $0.043 \pm 0.004$ & $0.034 \pm 0.004$ & $0.028 \pm 0.003$ \\
     & $\mathcal{S}_2$  &  $0.151 \pm 0.039$ & $0.076 \pm 0.023$ & $0.045 \pm 0.004$ & $0.036 \pm 0.004$ & $0.030 \pm 0.003$ \\
      & $\mathcal{S}_3$  & $0.159 \pm 0.039$ & $0.081 \pm 0.023$ & $0.050 \pm 0.004$ & $0.040 \pm 0.004$ & $0.034 \pm 0.003$ \\
      \hline
      \hline
    & Sample Mean  & $0.152 \pm 0.069$ & $0.066 \pm 0.029$ & $0.055 \pm 0.025$ & $0.034 \pm 0.015$ & $0.026 \pm 0.012$\\
     & Geometric Median  &  $0.137 \pm 0.062$ & $0.063 \pm 0.028$ & $0.047 \pm 0.021$ & $0.032 \pm 0.014$ & $0.025 \pm 0.011$ \\
      & Adversarial Estimator  & $0.472 \pm 0.545$ & $0.386 \pm 0.555$ & $0.346 \pm 0.546$ & $0.282 \pm 0.521$ & $0.206 \pm 0.455$\\
    \end{tabular}
    \caption{Mean squared test error for robust mean estimation among the finite model instantiations and baselines.}  
    \label{tab:robust_big_n}
\end{table*}

\end{document}